%% file: main.tex
\documentclass[11pt]{article}
\usepackage{template}
\usepackage{todonotes}

\title{Graph matching between bipartite and unipartite networks: to collapse, or not to collapse, that is the question.}

\author[$1,2$]{Jes\'us Arroyo}
\author[$2$]{Carey E. Priebe}
\author[$1$]{Vince~Lyzinski}

\affil[$1$]{\small Department of Mathematics, University of Maryland}
\affil[$2$]{\small Department of Applied Mathematics and Statistics, Johns Hopkins University}

\begin{document}

\maketitle

\begin{abstract}
    Graph matching consists of aligning the vertices of two unlabeled graphs in order to maximize the shared structure across networks; when the graphs are unipartite, this is commonly formulated as minimizing their edge disagreements. 
    In this paper we address the common setting in which one of the graphs to match is a bipartite network and one is unipartite. 
    Commonly, the bipartite networks are collapsed or projected into a unipartite graph, and graph matching proceeds as in the classical setting. This potentially leads to noisy edge estimates and loss of information. 
    We formulate the graph matching problem between a bipartite and a unipartite graph using an undirected graphical model, and introduce methods to find the alignment with this model without collapsing. We theoretically demonstrate that our methodology is consistent, and provide non-asymptotic conditions that ensure exact recovery of the matching solution.
    In simulations and real data examples, we show how our methods can result in a more accurate matching than the naive approach of transforming the bipartite networks into unipartite, and we demonstrate the performance gains achieved by our method in simulated and real data networks, including a co-authorship-citation network pair, and brain structural and functional data.
\end{abstract}

\input{paper-content}

\section*{Acknowledgements}
This material is based on research sponsored by the Air Force Research
Laboratory and DARPA under agreement number FA8750-18-2-0035 and FA8750-20-2-1001. The
U.S. Government is authorized to reproduce and distribute reprints for Governmental purposes notwithstanding any copyright notation thereon. The
views and conclusions contained herein are those of the authors and should
not be interpreted as necessarily representing the official policies or endorsements, either expressed or implied, of the Air Force Research Laboratory and
DARPA or the U.S. Government. Vince Lyzinski also gratefully acknowledge
the support of NIH grant BRAIN U01-NS108637. The authors thank Ross Lawrence and Joshua Vogelstein for their help in obtaining the MRI data.
\bibliographystyle{apa}
\bibliography{biblio2}

\input{appendix-content}

\end{document}

%% file: paper-content.tex
\section{Introduction}

%\IEEEPARstart{T}
The problem of inferring the correspondence between the vertices of two or more graphs, commonly referred as graph matching, has received recent interest in the literature, motivated by multiple applications in social and biological network analysis, image and document processing, pattern recognition, among others \citep{ConteReview,foggia2014graph,yan2016short}. 
Myriad approximate matching methods and random graph models have been developed for studying this problem, mostly focusing on the setting where the practitioner seeks to align a pair of unipartite graphs that represent two sets relationships between a common set of entities \citep{pedarsani2011privacy,narayanan2009anonymizing,JMLR:v15:lyzinski14a,rel,korula2014efficient,8039503,feizi2019spectral,Fan2019}. 
In many scenarios, the available network structures are more complex and can consists of different classes of vertices and/or edges that cannot be well represented using unipartite graphs.
In these settings, often the existing approaches are inapplicable directly, and significant data processing is needed to rend the problem amenable to existing analysis. 

The goal of this paper is to study the graph matching problem where the pair of graphs consist of one unipartite and one bipartite network. 
Bipartite networks are often used to encode relationships between entities in two different classes, for example, transactions of customers  with businesses \citep{bennett2007netflix}, authorship of publications by scientist \citep{ji2016coauthorship}, occurrence of words within a set of documents \citep{dhillon2001co}, protein-gene interactions \citep{pavlopoulos2018bipartite}, among many other examples. 
A convenient approach when dealing with bipartite data is to collapse into a unipartite network over one class of entities by defining some edges using a measure of connectivity, for example  one-mode projections \citep{Zhou2007,arora2012learning}, correlation \citep{macmahon2013community}, inverse covariance \citep{narayan2015sample,lo2017inferring}, or learning dependency graph structures \citep{she2019indirect}. Although by doing this it is possible to proceed with inference using methods for unipartite networks, this approach has multiple disadvantages. 
Having to estimate these edges requires the definition of a proper measure, which is difficult to assess in practice and might be different depending of the inference goals. 
Moreover, this  process results in noisy estimates that can induce  errors in subsequent inference.
Additionally, this estimation process often requires algorithmic and parameter decisions, and the results of subsequent analysis can be sensitive to those choices. 
Developing inference methods for bipartite network data directly is an important problem to address, and the analysis of bipartite network data has attracted significant interest, particularly in problems such as community detection \citep{Larremore2014,Zhou2019,Razaee2019},
co-clustering  \citep{dhillon2001co,Choi2017} and topic modeling \citep{blei2003latent}, for which methods specifically tailored for this type of data have been developed.

In particular, here we consider the setting when one of the graphs is a bipartite network between two disjoint sets of vertices $U$ and $V$, and the other graph is a unipartite graph with edges between elements in $U$. An example of this setting is given in Figure~\ref{fig:statsnet} with data from \citep{ji2016coauthorship}, in which $U$ represents a group of statisticians and $V$ is a set of papers written by them. 
An edge in the unipartite graph connects two statisticians if they are coauthors in one paper, and the edges of the bipartite graph link the authors with all the papers that they have cited. 
Typically, joint graph inference across such a network pair usually requires a priori knowledge of the correspondence between the vertices of the graphs.
In the more traditional setting when both graphs are unipartite (or both bipartite \citep{Koutra2013}), if the correspondence is unknown then graph matching is commonly used to recover an estimate of the latent correspondence before proceeding with graph inference \citep{chen2016joint}.

The problem of graph matching between bipartite and unipartite networks presents a more challenging setting than the classical graph matching problem, and although bipartite graphs are ubiquitous, to the best of our knowledge this  setting has not been studied before. Formulating the problem is not straightforward because the edges do not have a direct correspondence between the graphs, and in practice, bipartite graphs are typically collapsed into a unipartite graph before performing graph matching to solve this problem.
To address this problem, we introduce a new formulation of graph matching based on a novel joint model for the pair of graphs. We model the edges of the bipartite network using an undirected graphical model based on the information contained in a permuted version of the unipartite graph. 
The unshuffling permutation and the parameters of the model are then jointly estimated, yielding (demonstrably) superior performance over methods that perform this estimation in different stages. 
Moreover, using simulations and real data we show that our approach yields superior performance over classical graph matching strategies that first collapse the bipartite network onto one of its parts. 

%For example, if matching a friendship network to a purchasing network, it need not be the case that two vertices that are friends need have correlated purchasing habits.

\begin{figure}
    \centering
    \includegraphics[width=0.4\textwidth]{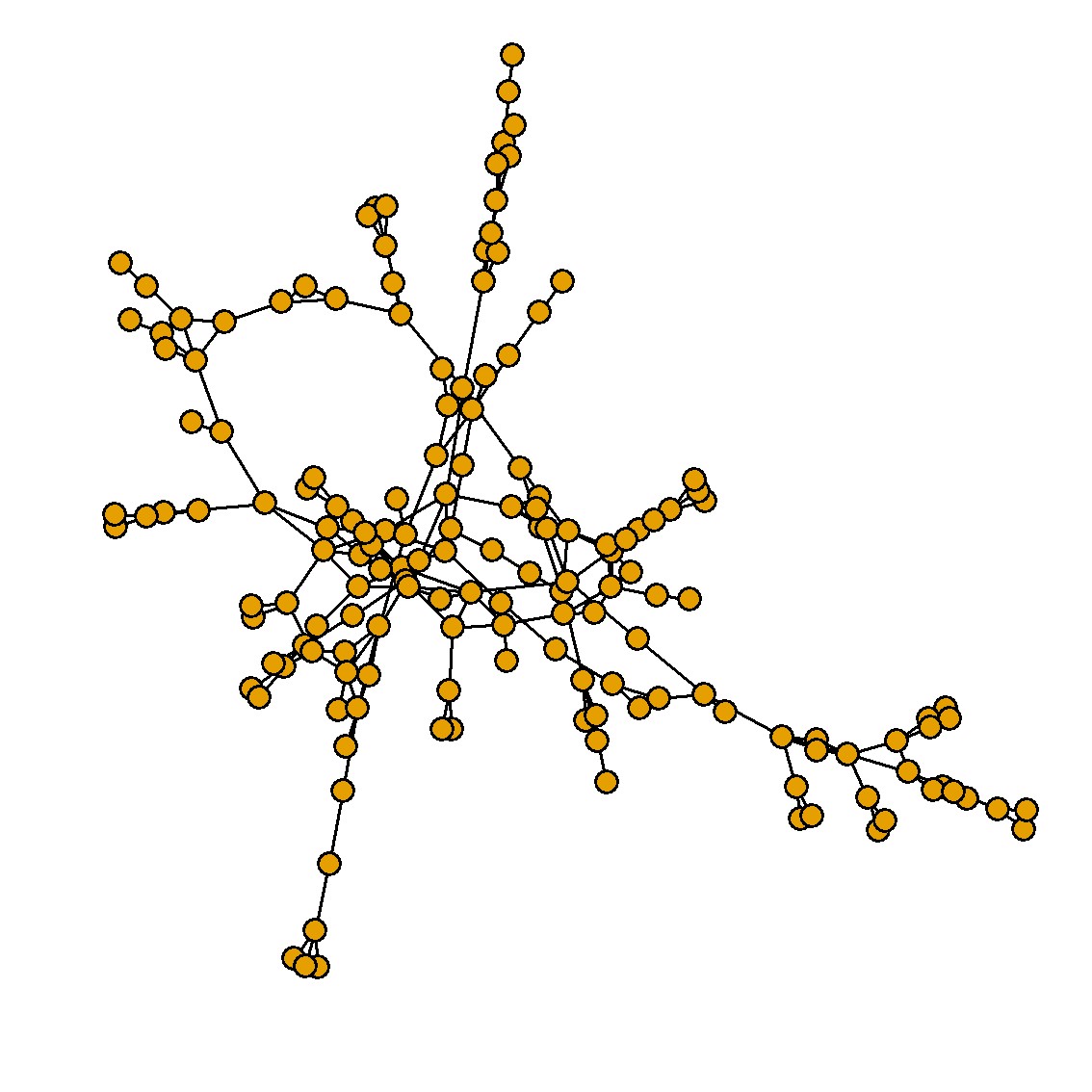}
    \includegraphics[width=0.4\textwidth]{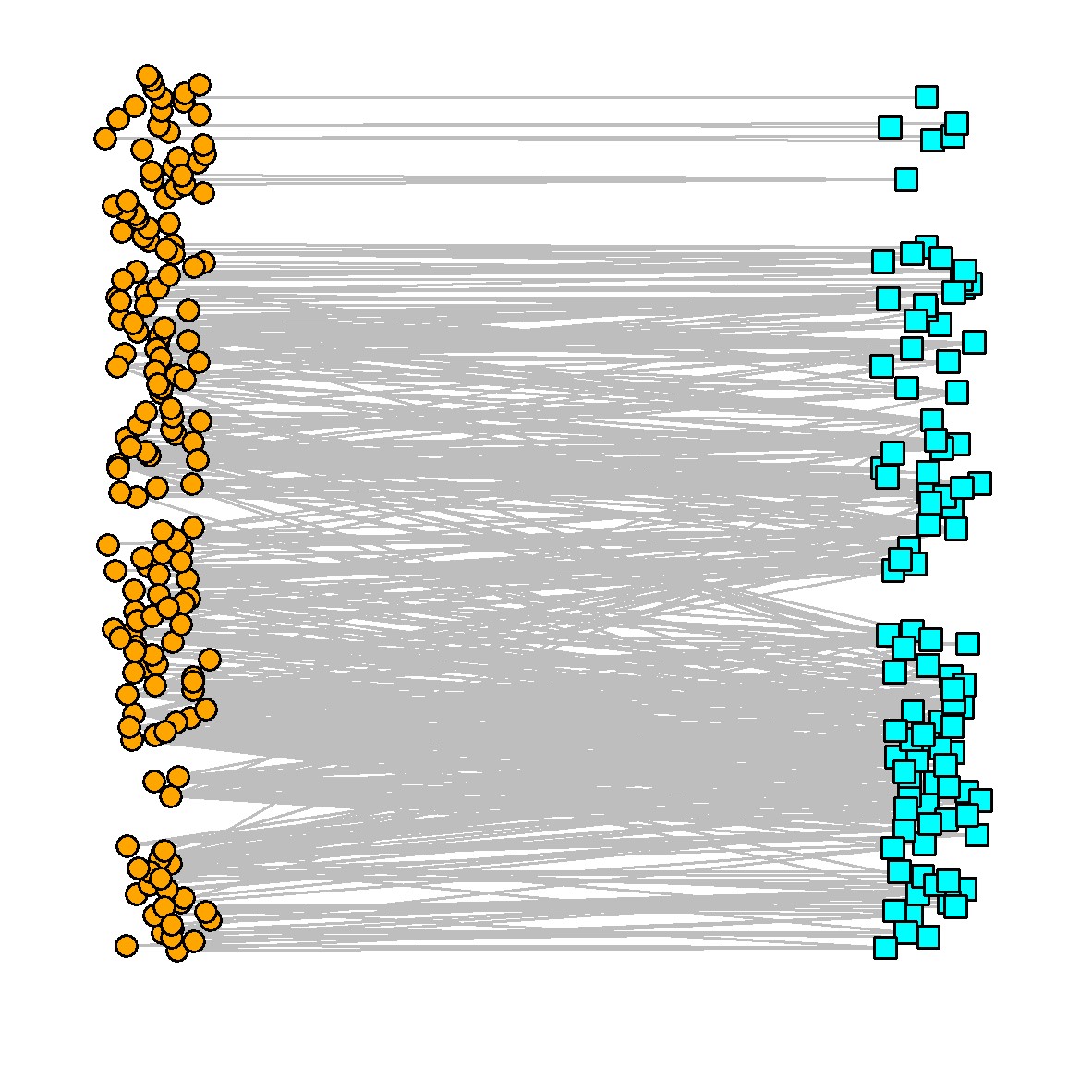}
    \caption{Co-authorship graph between authors (top) and bipartite graph of papers cited by each author (bottom) from top statistics journals \citep{ji2016coauthorship}. Authors are represented as orange circles in both graphs, and research papers correspond to teal squares.}
    \label{fig:statsnet}
\end{figure}

\section{Graph matching problem formulation}
Given a pair of simple graphs $G_1=(U_1, E_1)$ and $G_2=(U_2, E_2)$ with the same number of vertices $n:=|U_1|=|U_2|$ and edge sets $E_i\subset U_i\times U_i, i = 1,2$, the classical graph matching problem consists in finding a bijection between the vertices of the two graphs to maximize a measure of their induced similarity. 
If the graphs are isomorphic, then there exists an exact matching that makes the aligned edges identical, but usually, in practice, there is no such isomorphism (and finding it if it exists is notoriously difficult in practice \citep{babai2015graph}). 
The inexact graph matching aims to minimize the number of edge disagreements between the aligned graphs, and can be formally defined using the adjacency matrices $A_1,A_2\in\{0,1\}^{n\times n}$ of $G_1$ and $G_2$ respectively, as solving a quadratic assignment problem to find the permutation $\widehat{P}$ such that
\begin{equation}
    \widehat{P}\in\argmin_{P\in\Pi_n}\|A_1 - PA_2P^\top \|_F^2,
    \label{eq:GMP}
\end{equation} 
where $\Pi_n:=\{P\in\{0,1\}^{n\times n}: PP^\top =I_n\}$ is the set of $n\times n$ permutation matrices, with $I_n$ the $n\times n$ identity matrix, and $\|\cdot\|_F$ is the Frobenius norm. 
Typically, it is assumed that the vertex sets $U_1$ and $U_2$ represent the same set of entities $U$, and the graphs $G_1$ and $G_2$ represent their relationships in two different views, so the edges of the aligned graphs would be similar/correlated.
The formulation in Equation~\eqref{eq:GMP} aims to capture this similarity, and has been justified in different random graph models \citep{pedarsani2011privacy,rel,lyzinski2016information,cullina2016improved,cullina2017exact,Arroyo2018,Fan2019}.

In this paper, we consider a different setting, in which instead of observing two  views of the interactions between the vertices in $U$, in one of the views we observe the interaction between $U$ and a different set of vertices $V$, with $|V|=m$, which are encoded in a bipartite graph. Denote by $G=(U_1,E)$ and $H=(U_2,V,F)$ the unipartite and bipartite graphs that are observed, where $U_1$ and $U_2$ are two different orderings of $U$, and $E\subset \binom{U_1}{2}$ and $F\subset U_2\times V$ their corresponding edges. 
 
Define $A\in\{0,1\}^{n\times n}$ and $B\in\real^{n\times m}$ as the adjacency matrices for $G$ and $H$ respectively, so $A_{ij}=1$ indicates that $\{i,j\}\in E$, for $u_i,u_j\in U_1$, and $B_{ik}=1$ indicates an edge between $u_i\in U_2$ and $v_j\in V$ (or in general, $B_{ij}\in\real$ represents the weight of the relationship between these vertices). 
The goal of this paper is to find a  bijection between the vertices $U_1$ and $U_2$ that uncovers the latent correspondence between these vertex sets.

\subsection{The bipartite--to--unipartite graph matching problem}

Formulating a graph matching problem between a unipartite and a bipartite graph in the setting just described is not immediate, as this problem is not immediately compatible to the traditional graph matching formulation in Equation~\eqref{eq:GMP}.
Indeed, in the bipartite--to--unipartite setting it is not possible to define a correspondence between the edges of both graphs.  
Similarly, this setting is also different from other formulations of graph matching across two bipartite graphs \citep{Koutra2013}. 

A popular approach when dealing with bipartite graphs is to collapse the graph into a (possibly weighted) unipartite graph with a new adjacency matrix $\widetilde{B}\in\real^{n\times n}$ that uses an appropriate measure to define the edges. 
This is commonly done via \emph{one-mode projections} \citep{Zhou2007,Zweig2011}, such as the co-occurrence matrix $\widetilde B := \frac{1}{m}BB^\top $. 
Once this matrix is constructed, one can proceed in matching the two matrices $A$ and $\widetilde{B}$ by minimizing $\|A-P\widetilde{B}P^\top \|_F^2$ as in Equation~\eqref{eq:GMP}. However, this approach  faces several issues. On one hand, the matching solution depends on the collapsing method employed, and while multiple approaches for collapsing a graph are commonly used, none of them are constructed having the goal of graph matching in mind. 
Choosing the appropriate projection of the bipartite network among different options available, sometimes with parameter choices, is a complicated problem, 
and it is difficult to know a priori if a projection method is appropriate for subsequent graph matching. 
Finally, the estimated edges are amplifying noise in the original $B$, which makes the graph matching problem yet more challenging. 

To formulate the unipartite to bipartite matching problem directly, we need to introduce some information in common on the graphs $A$ and $B$, and we do so by considering a model for the edges of the bipartite graph as a function of $A$. 
We posit $B$ to be a random bipartite graph with a distribution depending on $A$ via an undirected graphical model \citep{Lauritzen1996,Wainwright2008}. 
Let $P^\ast\in\Pi_n$ be the permutation matrix associated with the mapping of  the vertices in $U_1$ to the  vertices in $U_2$, and let $W:= W(A,P^\ast) = (P^\ast)^\top  AP^\ast$ be the adjacency matrix obtained by permuting $A$ to match the order of the vertices in $B$. We model the columns of $B$ with a distribution that forms a Markov random field (MRF) with respect to $W$. Formally, a positive probability distribution $F$ on $\mathbb{R}^n$ forms a MRF  with respect to $W$ if a random vector $X\sim F$ satisfies the local Markov property of conditional independence
\begin{equation*}
    X_{i} \independent  X_{[n]\setminus \left\{\mathcal{N}_i(W)\cup\{i\}\right\}}\  | \ X_{\mathcal{N}_i(W)}, \quad\quad \forall i\in[n],
\end{equation*}
where $\mathcal{N}_i(W)=\{j\in[n]:W_{ij}=1\}$ is the set of neighbors of vertex $i$ in $W$. In other words, the distribution of $X_i$ can only be directly affected by other entries of $X$ that are connected to $i$ in the underlying graph $W$. We assume that the columns of $B$ are independent and identically distributed with distribution $F$, denoted by $B_1, \ldots, B_m\sim F$, where $B_k\in\real^n$ represents the $k$-th column of $B$, and we write the likelihood of $B$ as
\begin{equation}
    f(B) := \prod_{k=1}^m f_X(B_k). \label{eq:B-likelihood}
\end{equation}

If the graphical model $W$ is known, the (exact) graph matching problem in the MRF model just introduced can be formulated as finding the unshuffling permutation $P$ that aligns $A$ and $W$. In practice, however, $W$ is unobserved and some estimation process needs to be incorporated into the graph matching formulation. To make the problem more tractable, we focus in a subclass of graphical models in which the conditional distribution of each node can be expressed as a generalized linear model (GLM) \citep{mccullagh1989generalized,Yang2012}.
For a vector $X\sim F$ , the conditional distribution of $X_i=x_i$ given the rest of the variables is expressed as
\begin{align}
    %f_X\left(x_{i}\ |\  x_{[n]\setminus \{i\}}\right) & \propto  \exp\Big( \beta_i x_i + \sum_{j\in\mathcal{N}_i(W)} \Theta_{ij} x_i x_j \notag\\
    %&\quad \quad\quad\quad-2\Theta_{ii} C(x_i) \Big), \label{eq:GLM-node}
    f_X\left(x_{i}\ |\  x_{[n]\setminus \{i\}}\right) & \propto  \exp\Big( \beta_i x_i + \sum_{j\in\mathcal{N}_i(W)} \Theta_{ij} x_i x_j -2\Theta_{ii} C(x_i) \Big), \label{eq:GLM-node}
\end{align}
where $C(\cdot)$ is a known function, and $\gamma, \beta\in\real^n$ and $\Theta\in \real^{n\times n}$ are parameters of the model. Because of the local Markov property, only the entries $(i,j)$ of $\Theta$ for which $W_{ij}=1$ participate in the model, so by imposing the constraint
\begin{equation}
    \Theta_{ij}(1-W_{ij}) = 0,\quad\quad i,j\in[n], i\neq j, 
    \label{eq:Theta-constraint}
\end{equation}
which enforces the entries of $\Theta$ to be zero if $W_{ij}=0$, the likelihood of the model for a given $x\in\real^n$ can be expressed as
\begin{equation}
    f_X(x)\propto \exp\Big(\beta^\top  x + \sum_{i\neq j} \Theta_{ij} x_ix_j -2 \sum_{i=1}^n \Theta_{ii} C(x_i)\Big).
    \label{eq:GLM-full}
\end{equation}

The  model considered above introduces a dependence structure between edges in the bipartite graph, such that for any given vertex in $V$, the distribution of edges connecting to $U$ in  $B$ depends on the network between the vertices in $U$ encoded in $A$. In the context of the example given in Figure~\ref{fig:coauth}, our model considers that for any given paper, the likelihood that the paper is cited by  statistician $i$ is only potentially affected by the decisions made by other statisticians that are coauthors with $i$ (Equation~\eqref{eq:GLM-node}), and this assumption is enforced by the constraint in Equation~\eqref{eq:Theta-constraint}. 

The MRF framework allows us to model different edge distributions in the bipartite graph, and thus allows to handle weighted or unweighted networks by choosing an appropriate distribution. In this paper we focus on the following two special cases.

\begin{itemize}
    \item \textbf{Binary Bernoulli distribution.} The \emph{Ising model} \citep{ising1925beitrag} is popularly used in modeling binary-valued data.  
    The distribution of this model can be obtained by setting $C(X_i) = 0$, in which case, the node-wise distribution in Equation~\eqref{eq:GLM-node} takes the form of a logistic regression model, and the joint distribution of a binary random vector $X$ taking a value $x\in\{0,1\}^n$ is
    \begin{equation}
        \p(X=x) := \frac{1}{Z(\Theta, \beta)} \exp\Big(\sum_{i=1}^n \beta_i x_i + \sum_{i\neq j} \Theta_{ij} x_{i} x_{j} \Big), \label{eq:isingmodel}
    \end{equation}
    where $Z( \Theta, \beta)$ is the partition function or normalizing constant, given by
    \begin{equation}
        Z(\Theta,\beta) := \sum_{y\in\{0,1\}^n} \exp\Big(\sum_{i=1}^n \beta_i y_i + \sum_{i\neq j} \Theta_{ij} y_{i} y_{j}\Big).\label{eq:partfunc}
    \end{equation}
    The above expression requires the  calculation of the sum over all the possible elements in $\{0,1\}^n$, which has cardinality $2^n$, and hence it is not computable for a general $\Theta$. It is therefore sometimes more convenient to work with the pseudo-likelihood \citep{Bhattacharya2018,Ravikumar2010}, which is defined as the product of the conditional likelihoods of each variable, given by
    \begin{equation}
        \widetilde f_X(x) = \prod_{i=1}^n \Big[1 + \exp\big(-x_i\big(\beta_i + \sum_{i\neq j} \Theta_{ij}x_j\big) \big)\Big]^{-1}.
        \label{eq:Ising-pseudolikelihood}
    \end{equation}
    \item \textbf{Weighted Gaussian distribution.} By setting $C(X_i) = X^2_i$, the entries of $X$ can be modeled using a multivariate Gaussian distribution via a Gaussian graphical model \citep{Speed1986}.  The matrix $\Theta$ is  required to be positive definite, denoted by $\Theta\succ 0$,  so $\Theta^{-1}$ is a proper covariance matrix. Set $\mu=\Theta^{-1}\beta$. We denote $X\sim\mathcal{N}(\mu,\Theta^{-1})$ to the multivariate Gaussian distribution with the likelihood given by
    \begin{align}
        % f_X(x) = & \left(2^n\pi^n\det\Theta^{-1}\right)^{-1/2}\times \notag \\
        % &  \quad\exp\left( -\frac{1}{2}(x - \mu)^\top \Theta(x-\mu)\right), \label{eq:gaussian-likelihood}
         f_X(x) = & \left(2^n\pi^n\det\Theta^{-1}\right)^{-1/2}\times \exp\left( -\frac{1}{2}(x - \mu)^\top \Theta(x-\mu)\right), \label{eq:gaussian-likelihood}
    \end{align}
    and this matrix is constrained to satisfy Equation~\eqref{eq:Theta-constraint} and $\Theta \succ 0$. Observe that these constraints can be simultaneously satisfied for an arbitrary matrix $W$ since the diagonal entries of $\Theta$ are not constrained and the off-diagonals are never enforced to be non-zero.
\end{itemize}

Having defined the distribution of the edges in the bipartite graph, we now proceed to formulate the graph matching problem in this setting. The model imposes constraints in $\Theta$  that depend on $W$ according to Equation~\eqref{eq:Theta-constraint}, and this matrix $W=(P^\ast)^\top  A P^\ast$ is a function of the unknown permutation $P^\ast$. Therefore, it is natural to define the solution of the graph matching problem as the permutation matrix $\widehat{P}$ that constrains the zero-entries of the matrix $\Theta$ in such a way that it produces the best fit to the bipartite graph $B$ over all possible permutations $P\in\Pi_n$. The fit to the data can be measured using a loss function $\hat{\ell}(\Theta, \beta)$, such as the log-likelihood of $B$
\begin{align}
    % \hat{\ell}(\Theta, \beta) = &\frac{1}{m}\sum_{k=1}^m\Big(\beta^\top B_k + \sum_{i\neq j}\Theta_{ij}B_{ik}B_{jk} -2\sum_{i=1}^n\Theta_{ii}C(B_{ik})\Big)\notag\\
    % & \quad\quad   - \log(Z(\Theta, \beta)),\label{eq:likelihoodfunction}
     \hat{\ell}(\Theta, \beta) = &\frac{1}{m}\sum_{k=1}^m\Big(\beta^\top B_k + \sum_{i\neq j}\Theta_{ij}B_{ik}B_{jk} -2\sum_{i=1}^n\Theta_{ii}C(B_{ik})\Big)   - \log(Z(\Theta, \beta)),\label{eq:likelihoodfunction}
\end{align}
and the graph matching problem from a unipartite to a bipartite graph can then be defined as
\begin{equation}
\begin{aligned}
\left(\widehat{P}, \widehat{\Theta}, \widehat{\beta}\right) =  & & \argmax_{P, \Theta, \beta} & \quad \hat{\ell}(\Theta,\beta) \label{eq:GMP-uni-bip}\\
& &\textrm{subject to} & \quad \Theta_{ij} (1 - (P^\top AP)_{ij}) = 0, \quad  i\neq j,\\
& & & \quad P \in \Pi_n, \quad \Theta\in\real^{n\times n}, \quad\beta\in\real^n.
\end{aligned}    
\end{equation}
% \begin{equation}
% \begin{aligned}
% \left(\widehat{P}, \widehat{\Theta}, \widehat{\beta}\right) =  & & \argmax_{P, \Theta, \beta} & \quad \hat{\ell}(\Theta,\beta) \label{eq:GMP-uni-bip}\\
% & &\textrm{subject to} & \quad \Theta_{ij} (1 - (P^\top AP)_{ij}) = 0, \quad  i\neq j,\\
% & & & \quad P \in \Pi_n, \quad \Theta\in\real^{n\times n}, \quad\\
% & & & \beta\in\real^n.
% \end{aligned}    
% \end{equation}
The variables $\Theta$ and $\beta$ are nuisance parameters for graph matching, but as a by-product, this problem formulation also produces estimates for these parameters in the model.

\subsection{Inexact graph matching}
In the previous formulation, we assumed that the edges of $A$ and the underlying graphical model for $B$ can be perfectly matched, so, in a sense, the  formulation of the graph matching problem in Equation~\eqref{eq:GMP-uni-bip} corresponds to an exact matching.
In some situations, the edges of $A$ are noisy, and the correspondence with the edges of $W$ might not be exact. Additionally, the edges of $A$ can be weighted, and the magnitude of those weights can be associated with the magnitude of the corresponding entries in $\Theta$. In those scenarios, it is more appropriate to consider a model for $A$ that is parameterized by the graphical model structure $W$ and the weights in $\Theta$. In this section, we demonstrate how  to formulate a inexact graph matching when the entries of $A$ are an noisy version of the graphical model.

To account for noise, the graph unweighted $A$ can be modeled as an errorful observation of the underlying unshuffled graph $P^\ast W(P^\ast)^\top $, in which the edges are independently perturbed with some probability $p\in[0,1]$. This model for $A$ is common in the graph matching literature \citep{pedarsani2011privacy,korula2014efficient,JMLR:v15:lyzinski14a, Arroyo2018}, as it allows to measure the effect of the similarity or correlation between the edges of the graphs to be matched. Given $p$, $W$ and $P$, the likelihood of an adjacency matrix $A\in\{0,1\}^{n\times n}$ can be written as
\begin{equation}
    \p(A) = \prod_{i>j} (1-p)^{\mathbbm{1}\{A_{ij}= (PWP^\top )_{ij}\}} p^{\mathbbm{1}\{A_{ij}\neq (PWP^\top )_{ij}\}}. \label{eq:PA}
\end{equation}
To consider a weighted network $A$, the previous distribution can be extended to make the value of the edges of $A$ also depend on the magnitude of the corresponding entry in $\Theta$, but for simplicity we only focus on binary values of $A$ here.

By taking logarithms and arranging the terms in Equation~\eqref{eq:PA}, the log-likelihood can be written in a form that is more akin to the classical graph matching formulation~\eqref{eq:GMP}, given by
\begin{align*}
    \ell_A(W, P) & =  \log(p)\sum_{i>j}\mathbbm{1}\{A_{ij}\neq (PWP^\top )_{ij}\}\  + \log(1-p)\sum_{i>j}\mathbbm{1}\{A_{ij}=(PWP^\top )_{ij}\}\\
    & = -\lambda_p\|A-PWP^\top \|_F^2 + c_p,
\end{align*}
% \begin{align*}
%     \ell_A(W, P) & =  \log(p)\sum_{i>j}\mathbbm{1}\{A_{ij}\neq (PWP^\top )_{ij}\}\  +\\
%     & \quad \quad\log(1-p)\sum_{i>j}\mathbbm{1}\{A_{ij}=(PWP^\top )_{ij}\}\\
%     & = -\lambda_p\|A-PWP^\top \|_F^2 + c_p,
% \end{align*}
with $\lambda_p = \log((1-p)/p)$ and $c_p=\log(1-p)n(n-1)/2$. The constraint \eqref{eq:Theta-constraint} can be used here to relate $\Theta$ and $W$, and the graph matching problem is
\begin{equation}
\begin{aligned}
\left(\widehat{P}, \widehat{\Theta}, \widehat{\beta}\right) =  & & \argmax_{P, \beta, \Theta} & \quad \hat{\ell}(\Theta,\beta) - \lambda_p\|A-PWP^\top \|_F^2\label{eq:GMP-uni-bip-noise}\\
& &\textrm{subject to} & \quad W_{ij} = \mathbbm{1}\{\Theta_{ij}\neq 0\}, \quad \quad i\neq j,\\
& & & \quad P \in \Pi_n, \quad \Theta\in\real^{n\times n}, \quad \beta\in\real^n.
\end{aligned}    
\end{equation}
In this formulation, the edge perturbation probability $p$ can be seen as a penalty parameter that controls how similar the edges of the network $A$ and the graphical model $W$ are. Formulations \eqref{eq:GMP-uni-bip} and \eqref{eq:GMP-uni-bip-noise} are very similar, and in fact equivalent when $p=0$. In the following section, we present methods to approximate the solution of Equation~\eqref{eq:GMP-uni-bip}, but these can be extended to approximate the solution of formulation \eqref{eq:GMP-uni-bip-noise}.

%%%%%%%%%%%%%%%%%%%%%%%%%%%%%%%%%%%%%%%%%%%%%%%%%%%%%%%%%%%%%%%%%%%%%%%%%%%%%%%%%%%%%%%%%%%%%%%%%%%%%%%%%%%%%%%%%%%%%%%%%%%%%%%%%%%%%%%%%%%%%%%%

\section{Solving the unipartite to bipartite matching problem}

Under the graphical model for the bipartite network described before, the graph matching problem seeks to match the edges of $A$ with the non-zero entries of the graphical model for $B$. 
In principle, one can approach this matching problem by first collapsing the bipartite graph into an estimator of $\Theta$ or the set of edges of the graphical model $\{(i,j):\Theta_{ij}\neq 0\}$, and then matching the vertices of this estimator---which can be represented as a (weighted) unipartite graph---with the graph $A$. 
Estimating the edges of an undirected graphical model is a well studied problem, and multiple existing methods have addressed this estimation from different perspectives, both for  discrete  \citep{Ravikumar2010,Yang2012,Bresler2015} and
continuous \citep{Meinshausen2006,Banerjee2008,Friedman2008a,Rothman2008} graphical models. 
Although this two step procedure can succeed in some situations, there are multiple challenges that discourage this approach. 
Finding an accurate estimator of the edge set of a graphical model might require a large amount of signal, and the error induced by this noisy estimate can significantly affect the matching performance. 
In addition, methods require conditions on the topology of the graph that do not always hold \citep{Bento2009,Loh2013,Bresler2015}, yielding those estimators inconsistent in some circumstances. 
Often, there are tuning parameters involved in this estimation process, such as $\ell_1$ regularization approaches for variable selection, and choosing the right model can be a complicated task \citep{Meinshausen2010}.

In contrast, our approach to matching the vertices of graphs $A$ and $B$ is based on finding an approximate solution to the optimization problem~\eqref{eq:GMP-uni-bip}. 
This is motivated by the following theorem, which is a consequence of the consistency of the maximum likelihood estimator in Markov random fields (see for example \citep{Wainwright2008}). The proof can be found in the Appendix. Before introducing the theorem, we define $\ell^\ast(\Theta, \beta) :=\e[\hat{\ell}(\Theta, \beta)]$ as the expected likelihood, and
$\widehat{\Theta}_P$ and $\widetilde{\Theta}_P$ as the profile likelihood estimators for $\Theta$ given $P$ using the  empirical likelihood $\hat{\ell}$ and expected likelihood $\ell^\ast$, respectively (a formal definition of these quantities is included in the proof). We denote the smallest and largest eigenvalues of a matrix $M$ as $\gamma_{\min}(M)$ and $\gamma_{\max}(M)$, and $a\wedge b$ denotes the minimum of $a$ and $b$.

\begin{theorem}
\label{thm:main}
Let $A\in\{0,1\}^{n\times n}$ be the adjacency matrix of a simple unipartite graph, $P^\ast\in\Pi_n$ be a permutation of size $n$, and  $B^{(m)}\in\real^{n\times m}$ be a random  matrix with its columns distributed as i.i.d. samples from some graphical model as in Equation~\eqref{eq:GLM-node} with $W=(P^\ast)^\top A P^\ast$. Let $\widehat{P}_m$ be a solution of \eqref{eq:GMP-uni-bip}.
\begin{enumerate}
    \item Suppose that the columns of $B^{(m)}$ come from the Ising model or the Gaussian graphical model, with $\Theta^\ast\in\real^{n\times n}$ and $\beta^\ast\in\real^n$  the parameters of the model such that Equation~\eqref{eq:Theta-constraint} holds. 
If (i) $\ell^\ast(\Theta, \beta)$ has a unique minimizer at $(\Theta^\ast, \beta^\ast)$ and (ii) $\Theta_{ij}^\ast\neq 0$ for all $(i,j)$ with $W_{ij}=1$, then, as $m\rightarrow\infty$,
$$\|\widehat{P}^\top _mA\widehat{P}_m - (P^\ast)^\top A P^\ast\|_F^2\overset{\mathbb{P}}{\rightarrow} 0.$$
If $A$ does not have non-identity automorphisms, then $\widehat{P}_m\overset{\mathbb{P}}{\rightarrow}P^\ast$ as $m\rightarrow\infty$.

%%%%%%%%%%%%%%%%%%%%%%%%%%%%%%%%%%%%%%%%%%%%%%%%%%%%%%%%
\item Suppose that the columns of $B^{(m)}$ have a centered multivariate Gaussian distribution $B_i^{(m)}\overset{\text{i.i.d.}}{\sim}\mathcal{N}(0_n, (\Theta^\ast)^{-1})$, with $\Theta^\ast\in\real^{n\times n}$ some positive definite matrix  such that Equation~\eqref{eq:Theta-constraint} holds. Suppose that there exist positive constants $\alpha_1, \alpha_2$ with
\begin{equation}
    \alpha_1 \leq \min_{P\in\Pi_n}\gamma_{\min}(\widetilde\Theta_P), \quad \max_{P\in\Pi_n}\gamma_{\max}(\widetilde\Theta_P)\leq \alpha_2.
    \label{eq:theorem-condition-eigenvalues}
\end{equation}
Then, there are positive constants $C_1,C_2$ and $C_3$ that only depend on $\alpha_1$ and $\alpha_2$ such that if
\begin{equation}
    \left(\min_{P\neq P^\ast}\| \widetilde\Theta_P - \Theta^\ast\|_F^2\right)\wedge C_1 \geq C_2 \frac{(\|A\|_F^2 + n)\log n}{m}, \label{eq:thm-condition}
\end{equation}
then $\widehat{P}_m=P^\ast$ with probability at least $1-\frac{C_3}{n}$.

\end{enumerate}

\end{theorem}

The first part of the  theorem shows that the optimal permutation $\widehat{P}_m$ that solves \eqref{eq:GMP-uni-bip} converges in probability to a permutation in the automorphism group of $A$, and when this group contains a unique element then it also converges to the correct unshuffling permutation. One implication of this theorem is that regardless of the topology of the graph, this solution can guess the edges of the graphical model consistently. 

The second part of the theorem provides conditions on the model parameters under which the graph matching solution $\widehat{P}$ is correct with high probability, under the special case in which the columns of $B^{(m)}$ have a Gaussian distribution. 
The rate on the left hand side of Equation~\eqref{eq:thm-condition} coincides with the squared Frobenius error rate for estimating $\Theta^\ast$ using a $\ell_1$-penalized estimator
\citep{Rothman2008}. By exactly recovering $P^\ast$, our theory  ensures that the  edges of the graphical model are correctly estimated without requiring further irrepresentability conditions
\citep{Ravikumar2010}, but we keep in mind that in our setting we have access to a permuted version of the edges through $A$.

Because of the restriction of the support of $\widetilde\Theta_P$, Equation~\eqref{eq:thm-condition} can be simplified  to a slightly stronger condition that only depends on the smallest non-zero entry of $\Theta^\ast$, given by $\theta_{\min} := \min\{|\Theta^\ast_{ij}|:A_{ij}=1\}$. 
By observing that $2\|\widetilde\Theta_P-\Theta^\ast\|_F^2 \geq \theta_{\min}^2\|P^\top A P - A\|_F^2$, a sufficient condition is that
$$\min_{P\neq P^\ast}\|P^\top A P - A\|_F^2\geq C_4\frac{(\|A\|_F^2 + n)\log n}{\theta_{\min}^2m}$$
for some positive constant $C_4$. The value on the left-hand side quantifies the regularity of the graph $A$ by computing the number of edge disagreements after permuting some of its vertex labels, and it is related to the difficulty of graph matching in the classic unipartite setting \citep{rel,Arroyo2018}.

Unfortunately, solving this optimization problem  requires searching over all elements in $\Pi_n$, a discrete set of cardinality $n!$, which is computationally infeasible. Nevertheless, the space of all possible undirected graphical models has a size $2^{n(n-1)/2}$, that is orders of magnitude larger than $|\Pi_n|$, which suggests that graphical model estimation is a harder problem than graph matching in this context. 
We focus on finding an approximate solution to this problem next.

%Our approach to matching the vertices of graphs $A$ and $B$ is based on finding an approximate solution to the optimization problem~\eqref{eq:GMP-uni-bip}.  Unfortunately, solving this optimization problem  requires to search over all elements in $\Pi_n$, a discrete set of cardinality $n!$, which is computationally infeasible. Nevertheless, the space of all possible undirected graphical models has a size $2^{n(n-1)/2)}$, that is orders of magnitute larger than $|\Pi_n|$, which suggests that graphical model estimation is a harder problem than graph matching in this context. We focus on finding an approximate solution to this problem next.
%\todo{I commented out the theorem because there is a missing detail in the proof, if we fix it I'll put it back}

\subsection{Matching via collapsing the bipartite network\label{sec:profile}}
In light of Theorem \ref{thm:main}, solving problem \eqref{eq:GMP-uni-bip} provides a principled heuristic for aligning the vertices of $A$ to those in $B$.
Given this, our model also offers some insights into when collapsing $B$ using one mode projection is justified.
Recall the collapsed graph $\widetilde B =\frac{1}{m}BB^\top $  previously defined, which counts the co-occurrence between every pair of vertices. Under a restricted Ising model, the following theorem shows that the profile likelihood for $\Theta$ in the bipartite--to--unipartite matching problem \eqref{eq:GMP-uni-bip} is equivalent to a unipartite matching between $A$ and $\widetilde{B}$. Hence, under this restricted model, collapsing the bipartite network and then matching the graphs is an appropriate method. For the proof of this result, see Appendix.
\begin{theorem}
\label{thm:main2}
Suppose that in the Ising model for the bipartite graph defined in Equation~\eqref{eq:GLM-full}, the parameters satisfy $\beta=0$ and $\Theta_{ij}=\theta_0\mathbbm{1}\{((P^{\ast})^\top AP^\ast)_{ij}=1\}$  
for all $i,j\in[n]$, $i\neq j$, with $\theta_0\in\real$, and the graph $A$ is non-empty. Define $\hat{P}_{\text{MLE}}$ and $\hat{\theta}_{\text{MLE}}$ as the  maximum likelihood estimators for $\theta_0$ and $P^\ast$. % constrained to $\theta_0 >0$
We then have that 
\begin{align}
[i.]& \quad\text{If }\hat\theta_{MLE}>0, \text{ then }  \hat{P}_{\text{MLE}} = 
    \argmin_{P\in\Pi_n}\|A-P\widetilde{B}P^\top \|_F^2, \label{eq:mle-profile1}\\
    [ii.]& \quad\text{If }\hat\theta_{MLE}<0, \text{ then }  \hat{P}_{\text{MLE}} = 
    \argmax_{P\in\Pi_n}\|A-P\widetilde{B}P^\top \|_F^2,
    \notag\\
    [iii.]&\quad \text{If }\hat\theta_{MLE}=0, \text{ then }  \hat{P}_{\text{MLE}} = 
    \Pi_n,
    \notag
\end{align}

%\begin{equation}
%    \hat{P}_{\text{MLE}} =
%    \argmax_{P\in\Pi_n} \text{Tr}\left(P^\top AP\widetilde{B}\right) = \argmin_{P\in\Pi_n}\|A-P\widetilde{B}P^\top \|_F^2.        \label{eq:mle-profile1}
%\end{equation}

% \begin{equation}
%     \hat{P}_{\text{MLE}} = \left\{\begin{array}{ll}
%     %\argmax_{P\in\Pi_n} \text{Tr}\left(P^\top AP\widetilde{B}\right) =
%     \argmin_{P\in\Pi_n}\|A-P\widetilde{B}P^\top \|_F^2,     &  \text{if }\hat{\theta}_{\text{MLE}}>0,\\
%     %\argmin_{P\in\Pi_n} \text{Tr}\left(P^\top AP\widetilde{B}\right) =
%     \argmax_{P\in\Pi_n}\|A-P\widetilde{B}P^\top \|_F^2,
%          & \text{if }\hat{\theta}_{\text{MLE}}<0,\\
%     \text{any $P\in\Pi_n$} &  \text{if }\hat{\theta}_{\text{MLE}}=0.
%     \end{array}\right.
%     \label{eq:mle-profile1}.
% \end{equation}
\end{theorem}
\noindent Stated simply, when all the non-zero entries of $\Theta$ have the same magnitude and sign and depending on the properties of the MLE of $\theta$, estimating $P$ by collapsing prior to graph matching will yield a consistent estimate of $P$ (case i) or an inconsistent estimate (case ii).  Moreover, in the absence of an accurate estimate of $\theta$, there is no way to know from the data directly whether collapsing is the exact right pre-processing step.
We also note that the equivalence in Theorem \ref{thm:main2} is only true under a restricted model framework, but even under this model,  $\widetilde{B}$ and $A$ need not have the same support of non-zero entries, even for large $m$. This further complicates solving the optimization problem \eqref{eq:mle-profile1} in practice.

The equivalence in Theorem \ref{thm:main2} is true under a restricted model framework in which there is no individual vertex effect (controlled by $\beta$) and the strength of the connections is the same (i.e., the non-zero entries of $\Theta$ are all equal). One may wonder whether collapsing the networks using a one-mode projection is appropriate under more general settings. Figure~\ref{tab:collapsing-doesnt-work} presents empirical counterexamples suggesting that in general, this approach yields the incorrect edge structure, and thus a classical matching of $A$ and the collapsed graph will not recover the correct vertex alignment. In all cases, a bipartite graph was generated using the Ising model with the parameters described in the second column of the table by sampling $m=2000$ bipartite vertices to ensure a good estimation of the collapsed graphs $\widetilde{B}$. We additionally include the matrix of Pearson correlations between the rows of $B$ as an additional collapsing method, denoted by  $\widehat{\Sigma}$. 

The first example (first row of Figure~\ref{tab:collapsing-doesnt-work}) considers a setting in which all the non-zero elements of $\Theta^\ast$ are equal to the same constant, but the entries of $\beta^\ast$ are not equal to zero. The vertices corresponding to the largest elements of $\beta^\ast$ are more likely to independently connect to vertices on the bipartite graph, causing a higher likelihood of co-occurrence due to their high degree, and thus, a one-mode projection of the data suggests a strong relation between the first and third vertices, which are not connected on $A$, yielding an incorrect matching result. Nevertheless, in this scenario, the sample correlation matrix $\widehat{\Sigma}$ still reflects a higher connectivity between the first and second vertices. In the second example (second row of the table) we consider a setting in which differences in the values of $\Theta^\ast$ yield an incorrect matching solution. The vertices are partitioned into two groups ($\{1,2,3\}$ and $\{4,5,6\}$) that are disconnected on $A$, and the vertices in the second group are fully connected between each other. A smaller weight on $\Theta^\ast$ for the connections of the second group causes the collapsed graphs to highlight the connections on the first group regardless on whether the vertices are actually connected on $A$, and thus a classical graph matching with a collapsed graphs will result in aligning the vertices on different groups.

\begin{figure*}[ht!]
    \centering
    \begin{tabular}{c|c c|c|c}
    	Unipartite graph & \multicolumn{2}{|c|}{Bipartite  parameters}  & One-mode projection & Correlation  \\ 
    	 $A$ & $\beta^\ast$ &  $\Theta^\ast$ & $\widetilde{B}$  & $\widehat{\Sigma}$  \\ 
    	\hline 
    	\includegraphics[width=0.2\textwidth]{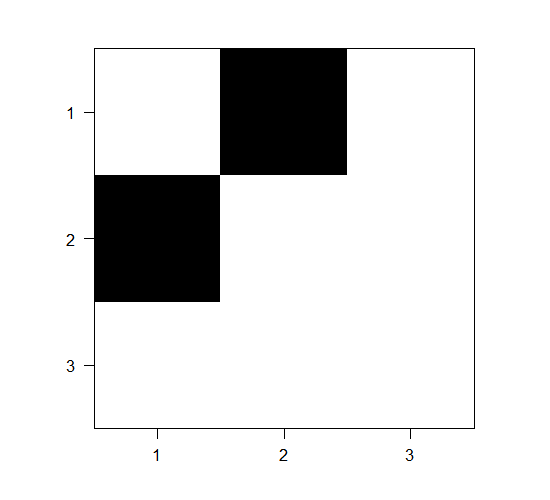} & \includegraphics[width=0.08\textwidth]{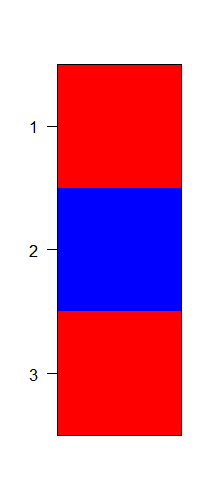} & \includegraphics[width=0.2\textwidth]{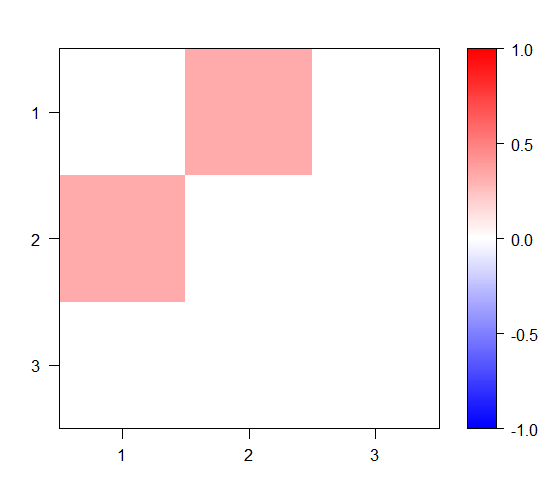} & \includegraphics[width=0.2\textwidth]{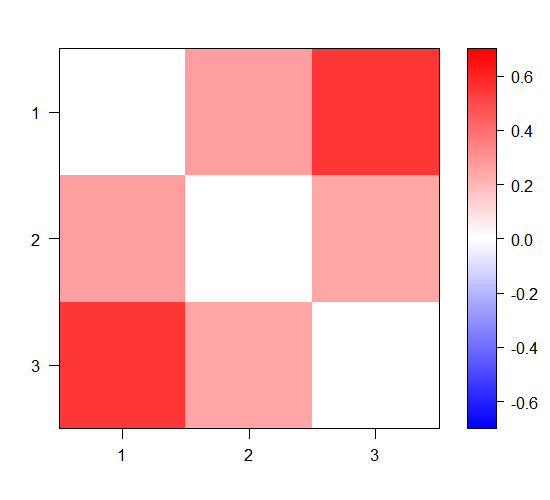}  &  \includegraphics[width=0.2\textwidth]{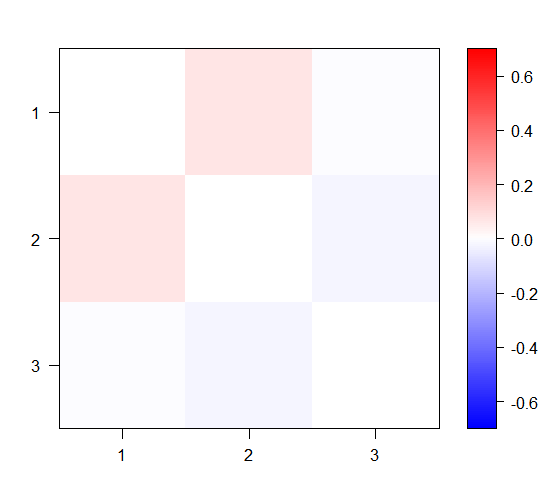}\\ 
    	\hline 
    	\includegraphics[width=0.2\textwidth]{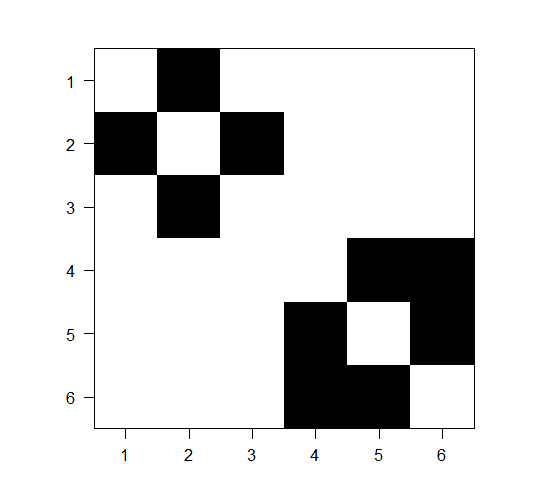} & \includegraphics[width=0.065\textwidth]{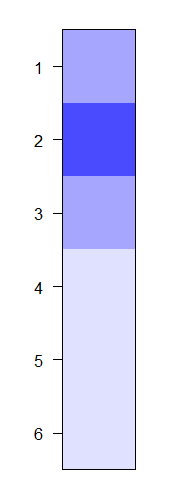} & \includegraphics[width=0.2\textwidth]{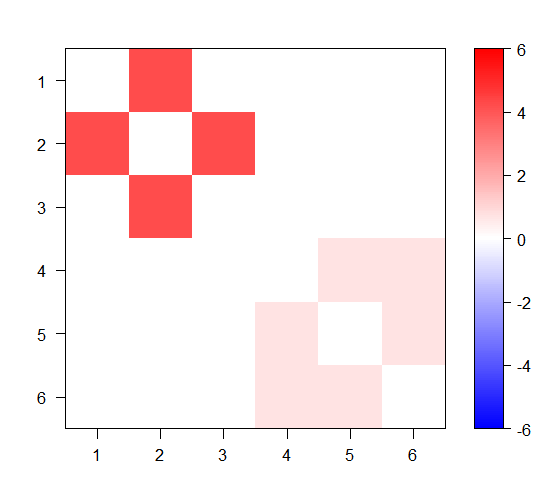} & \includegraphics[width=0.2\textwidth]{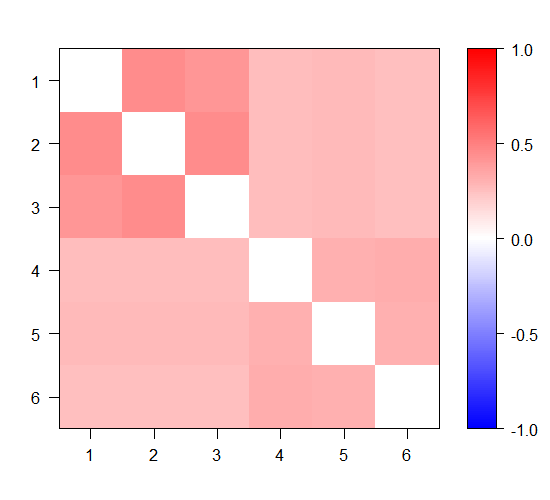}  &  \includegraphics[width=0.2\textwidth]{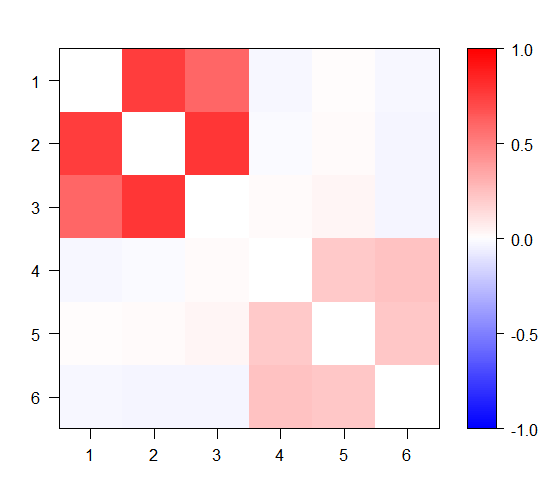}
    \end{tabular} 
    \caption{Empirical examples of cases when collapsing the bipartite graph to a one-mode projection or correlation matrix yields a different edge structure than the unipartite graph, and an incorrect  graph matching solution. The bipartite graphs are generated from the Ising model using the parameters shown on the second and third columns, and the collapsed bipartite graphs  based on one-mode projection  (fourth columns) or Pearson correlation (fifth column) are also shown. The first row illustrates that when $\beta$ is not zero, the effect of the vertex degree can create strong relations on the one-mode projection that are not present in the unipartite graph. The second row shows that the strength of the associations in $\Theta$ can result in spurious edges on the one-mode projection or correlation matrices.}
    \label{tab:collapsing-doesnt-work}
\end{figure*}

\subsection{Matching via inverse covariance estimation}
Computing the likelihood of the bipartite graph is challenging because  the calculation of the partition function is infeasible for some distributions, including for the Ising model. 
For a Gaussian graphical model however, the likelihood takes a simpler form, as observed in Equation~\eqref{eq:gaussian-likelihood}, and the MLE for $\mu$ is $\hat{\mu} = \frac{1}{m}\sum_{k=1}^mB_k$ (with $B_k$ being row $k$ of $B$), which does not depend on $P$, so the log-likelihood of the bipartite graph can be expressed only in terms of $\Theta$ as
\begin{equation}
    \hat{\ell}(\Theta) = \log \det\Theta - \Tr(\hat{\Sigma}\Theta) \label{eq:GMP-ggm-loss},
\end{equation}
 where  $\widehat{\Sigma} = \frac{1}{m}\sum_{k=1}^m(B_k - \hat{\mu})(B_k - \hat{\mu})^\top $ is the sample covariance of the rows of the bipartite incidence matrix (see for example \citep{Friedman2008a}). The bipartite--to--unipartite matching problem constrains $\Theta$ to satisfy  $\Theta_{ij}(1-(P^\top AP)_{ij})=0$ for $i,j\in[n]$.

Solving the graph matching optimization problem for both $P$ and $\Theta$ is computationally infeasible due to the combinatorial constraint set. 
On the other hand, if some $P$ is fixed, finding the profile MLE for $\Theta$, denoted as  $\widehat{\Theta}_P$, is a convex optimization problem and therefore can be solved efficiently. Using this profile MLE,  the fit of different permutations can be compared by evaluating $\hat{\ell}( \widehat{\Theta}_P)$.

To find an approximate solution of  \eqref{eq:GMP-uni-bip} for both $P$ and $\Theta$, we relax the graph matching problem  by replacing $P\in\Pi_n$ with a real-valued matrix $D\in\mathcal{D}_n$, where $\mathcal{D}_n$ denotes the set of $n\times n$ doubly stochastic matrices
$\mathcal{D}_n = \left\{D\in\real^{n\times n}:
D1_n=D^\top 1_n = 1, D_{ij}\geq 0 \right\}$.
This is effectively relaxing the permutation search space to its convex hull, and it is a common strategy in optimization problems over the space of permutations \citep{FAQ,rel,zhang2019kergm,fiori2013robust} since the solution typically lies on the extreme points of this polytope  \citep{maron2018probably}.
Additionally, the equality constraints in \eqref{eq:GMP-uni-bip} are replaced with a penalty function \citep{nocedal2006numerical} with a penalty parameter $\lambda >0$. We choose a non-smooth penalty based on the $\ell_1$-norm  to induce a sparse solution, so if $\lambda$ is sufficiently large, the equality constraint is enforced. However, as we will see later, a large value of $\lambda$ makes the optimization algorithm more likely to get stuck in local minima. The new optimization problem then becomes
\begin{equation}
    \begin{aligned}
    \left(\widehat{D}^{(\lambda)}, \widehat{\Theta}^{(\lambda)}\right) =  & & \argmax_{D,  \Theta} & \bigg\{\log \det \Theta - \Tr(\hat{\Sigma}\Theta) -  \lambda \sum_{i\neq j}\left|(1-(D^\top AD)_{ij})\Theta_{ij}\right| \bigg\}
    \label{eq:GMP-inverse-cov}\\
    & &\textrm{subject to} &  \quad D \in \mathcal{D}_n, \quad \Theta\in\real^{n\times n}, \quad \Theta \succ 0.
    \end{aligned}    
\end{equation}
% \begin{equation}
%     \begin{aligned}
%     \left(\widehat{D}^{(\lambda)}, \widehat{\Theta}^{(\lambda)}\right) =  & & \argmax_{D,  \Theta} & \bigg\{\log \det \Theta - \Tr(\hat{\Sigma}\Theta) - \\
%     & & & \lambda \sum_{i\neq j}\left|(1-(D^\top AD)_{ij})\Theta_{ij}\right| \bigg\}
%     \label{eq:GMP-inverse-cov}\\
%     & &\textrm{subject to} &  \quad D \in \mathcal{D}_n, \quad \Theta\in\real^{n\times n}, \quad \\
%     & & & \Theta \succ 0.
%     \end{aligned}    
% \end{equation}
The optimization problem above is closely related to lasso penalized problems for estimating the inverse covariance in a Gaussian graphical model \citep{Banerjee2008,Friedman2008a, Rothman2008}. In particular, for a fixed $D=\frac{1}{n}1_n1_n^\top $ (so no prior information about the correct matching is given) all off-diagonal entries of $\Theta$ are equally penalized.
On the other hand, if $D=P^\ast$, only the entries of $\Theta$ for which $\Theta^\ast_{ij}=0$ are penalized, which potentially results in better estimation of the graphical model when the matching of the vertices is performed correctly.

In order to solve~\eqref{eq:GMP-inverse-cov}, we use an alternating minimization strategy by fixing $D$ or $\Theta$, and solve the resulting problem to obtain a sequence of estimators $\{(\widehat{D}^{(t,\lambda)}, \widehat{\Theta}^{(t,\lambda)})\}_{t=1}^{T^\ast}$.
We iterate this until a maximum number of iterations $T^\ast$ is reached. 
The initialization value is set to $\hat{D}^{(1,\lambda)}=\frac{1}{n}1_n1_n^\top $. Optimizing $\Theta$ with a fixed $\widehat{D}^{(t,\lambda)}$ requires to solve a problem analogous to the graphical lasso, given by
\begin{equation}
\widehat{\Theta}^{(t,\lambda)} = \argmax_{\Theta\succ 0} \left\{ \log \det\Theta - \text{Tr}(\widehat{\Sigma}\Theta) - \lambda\sum_{i\neq j}\omega_{ij}^{(t,\lambda)} |\Theta_{ij}| \right\} , \label{eq:step-graphicallasso}
\end{equation}
with $\omega^{(t,\lambda)}_{ij} = 1-(\hat{D}^{(t,\lambda)^\top }A \hat{D}^{(t,\lambda)})_{ij}$. This is a convex optimization problem that can be efficiently solved by the method presented in \citep{Friedman2008a} %\cite{}
and implemented in the \texttt{glasso} R package. Optimizing $D$ given a fixed $\hat{\Theta}^{(t,\lambda)}$ results in a non-convex quadratic assignment problem 
\begin{equation}
\hat{D}^{(t,\lambda)} = \argmin_{D\in\mathcal{D}_n} \Tr\left(D^\top AD|\hat{\Theta}^{(t,\lambda)}|\right) , \label{eq:step-faq}
\end{equation}
with $|\hat{\Theta}^{(t,\lambda)}|$ the entrywise absolute value operator. We solve the problem above using the Fast Approximate Quadratic assignment problem (FAQ) algorithm \citep{FAQ,rel}, which uses the Frank-Wolfe methodology to obtain an approximate  local minimizer of \eqref{eq:step-faq}, and then  projects $\hat{D}^{(t,\lambda)} $ onto the set of permutations $\Pi_n$ to obtain a permutation $\hat{P}^{(t,\lambda)} $. This is implemented with the \texttt{IGraphMatch} package (\url{https://github.com/dpmcsuss/iGraphMatch}).
The alternating minimization is repeated until the solution has converged (measured by the change in the loss function \eqref{eq:GMP-ggm-loss} between consecutive iterations) or  $T^\ast$ iterations have been reached. The  parameter $\lambda$ controls the sparsity of the solution in \eqref{eq:step-graphicallasso}. To select an appropriate value, we repeat the procedure for a sequence of values $\lambda_1, \ldots, \lambda_{S^\ast}$, resulting in a set of solutions $\{\hat{P}^{(t,\lambda_s)}, t\in[T^\ast], s\in[S^\ast]\}$. The optimal permutation among this set is chosen by calculating the value of the loss function $\hat{\ell}(\widehat{\Theta}_P)$ in Equation \eqref{eq:GMP-ggm-loss}.  This process is summarized in Algorithm~\ref{alg:Bip-inv}.

The method outlined above can be understood as a matching between the optimal permutation and the non-zero elements of the population inverse covariance matrix of the rows of $B$. In a Gaussian distribution, this non-zero elements actually coincide with the edges of the undirected graphical model for $B$. This is not necessarily the case in non-Gaussian distributions, but penalized inverse covariance estimators sometimes still produce accurate results for estimating undirected graphical models even in those cases \citep{Banerjee2008,Loh2013}, and thus the matching solution can potentially succeed. Alternatively, another approach based on the appropriate distribution of the data is described next. 

\begin{algorithm}
 		\caption{Unipartite to bipartite matching via penalized inverse covariance estimation}
 		\begin{algorithmic} 
 			\Input Adjacency matrix $A$, incidence matrix $B$.
 			\FOR{each $\lambda \in \{\lambda_s\}_{s=1}^{S^\ast}$}
 			\State Initialize $\hat{D}^{(1,\lambda)}=\frac{1}{n}1_n1_n^\top $.
 			\FOR{$t=1, \ldots, T^\ast$, or until convergence}
 			\State Update $\hat{\Theta}^{(t,\lambda)}$ by solving \eqref{eq:step-graphicallasso}.
 			\State Update $\hat{D}^{(t+1,\lambda)}$ by solving \eqref{eq:step-faq}.
 			\State Set $\hat{P}^{(t+1,\lambda)}$ as the projection of $\hat{D}^{(t,\lambda)}$ into $\Pi_n$.
 			\ENDFOR
 			\ENDFOR
 			\State Choose the permutation with the largest value of $\hat{\ell}( \hat{\Theta}_P)$ among the permutations $P\in \{P^{(t,\lambda_s)}, s\in[S^\ast], t\in[T^\ast]\}$.
 			\Output Permutation $\widehat{P}$, inverse covariance estimate $\hat{\Theta}_{\widehat{P}}$. 
 		\end{algorithmic}
 		\label{alg:Bip-inv}
 	\end{algorithm}
 	
\subsection{A pseudo-likelihood approach}
Because of the intractability of the likelihood, an alternative approach is to substitute the loss function of problem~\eqref{eq:GMP-uni-bip} with the pseudo-likelihood. This approach is standard for graphical model learning in statistics and machine learning \citep{Meinshausen2006,Ravikumar2010}, and here we show how this can be used for graph matching as well. We follow a similar procedure as in the previous section to obtain an approximate solution to  the pseudo-likelihood problem. The method in this section is outlined for an Ising model, but the same approach can be followed for other distributions.

The log-pseudo-likelihood of the bipartite graph for an Ising model is given by
\begin{equation}
    \widetilde{\ell}(\Theta, \beta) = \sum_{j=1}^n\left\{\sum_{k=1}^m -\log \Big(1 + \exp \big[-B_{jk}\big(\beta_j + B_k^\top \Theta_{j}\big)\big]\Big)\right\}, \label{eq:pseudolikelihood}
\end{equation}
with $\Theta\in\real^{n\times n}, \text{diag}(\Theta)=0, \beta\in\real^n$,  $\Theta_j$ indicates the $j$-th column of $\Theta$ and $\beta_j$ indicates the $j$-th entry of $\beta$. As before, given a fixed value of $P$, the profile pseudo-likelihood estimators $\hat{\Theta}_P$ and $\hat{\beta}_P$ can be numerically obtained via convex optimization to compute  $\widetilde{\ell}(\hat{\Theta}_P, \hat{\beta}_P)$.

To find an approximate solution for $P$ in maximizing \eqref{eq:pseudolikelihood}, we follow the same process as in the previous section, by relaxing $P$ to be a doubly stochastic matrix and replacing the constraints with a penalty term. This results in an optimization problem analogous to \eqref{eq:GMP-inverse-cov}, with the inclusion of the variable $\beta$, and with $\Theta\in\real^{n\times n}$ only constrained to have zeros in the diagonal. 
The alternating optimization proceeds by maximizing $(\Theta,\beta)$  and $D$. 
Given a fixed value $\hat{D}^{(t,\lambda)}$, the resulting problem for $(\Theta,\beta)$ can be split into $n$ separate problems for $(\Theta_j,\beta_j),j\in[n]$, given by 
% \begin{align}
% & (\hat{\Theta}_j^{(t,\lambda)}, \hat{\beta}_j^{(t,\lambda)}) =\notag\\
% &\quad\quad\argmax_{\Theta_j\in\real^n,\beta_j\in\real} \bigg\{-\log \Big(1 + \exp \big[-B_{jk}\big(\beta_j + B_k^\top \Theta_{j}\big)\big]\Big) -\notag\\
% & \quad\quad\quad\quad\quad\quad\quad\quad\quad\quad\lambda\sum_{i=1}^n \omega_{ij}^{(t,\lambda)} |\Theta_{ij}| \bigg\} , \label{eq:step-logreglasso}
% \end{align}
\begin{align}
& (\hat{\Theta}_j^{(t,\lambda)}, \hat{\beta}_j^{(t,\lambda)}) =\argmax_{\Theta_j\in\real^n,\beta_j\in\real} \bigg\{-\log \Big(1 + \exp \big[-B_{jk}\big(\beta_j + B_k^\top \Theta_{j}\big)\big]\Big) -\lambda\sum_{i=1}^n \omega_{ij}^{(t,\lambda)} |\Theta_{ij}| \bigg\} , \label{eq:step-logreglasso}
\end{align}
with $\Theta_{jj}=0$. The optimization above is a lasso penalized logistic regression \citep{Friedman2010}, and the solution is implemented using the
\texttt{glmnet} R package \citep{friedman2009glmnet}. Algorithm~\ref{alg:Bip-pseudo} summarizes the process for estimating $P$ with the pseudolikelihood loss function.

\begin{algorithm}
 		\caption{Unipartite to bipartite matching via penalized pseudolikelihood}
 		\begin{algorithmic} 
 			\Input Adjacency matrix $A$, incidence matrix $B$.
 			\FOR{each $\lambda \in \{\lambda_s\}_{s=1}^{S^\ast}$}
 			\State Initialize $\hat{D}^{(1,\lambda)}=\frac{1}{n}1_n1_n^\top $.
 			\FOR{$t=1, \ldots, T^\ast$, or until convergence}
 			\FOR{$j=1,\ldots, n$}
 			    \State Update $(\hat{\Theta}^{(t,\lambda)}_j, \hat{\beta}^{(t,\lambda)}_j)$ by solving \eqref{eq:step-logreglasso}.
 			\ENDFOR
 			\State Update $\hat{D}^{(t+1,\lambda)}$ by solving \eqref{eq:step-faq}.
 			\State Set $\hat{P}^{(t+1,\lambda)}$ as the projection of $\hat{D}^{(t,\lambda)}$ into $\Pi_n$.
 			\ENDFOR
 			\ENDFOR
 			\State Choose the permutation with the largest value of $\widetilde{\ell}(\hat{\Theta}_P)$ among $P\in \{P^{(t,\lambda_s)}, s\in[S^\ast], t\in[T^\ast]\}$.
 			\Output Permutation $\widehat{P}$, estimated parameters $\hat{\Theta}_{\widehat{P}}$ and $\hat{\beta}_{\widehat{P}}$. 
 		\end{algorithmic}
 		\label{alg:Bip-pseudo}
 	\end{algorithm}
 	
\subsection{Seeded graph matching}
In some applications, a partial alignment of the vertices is known a priori, and the goal is to match the remaining vertices. This setting is known as \emph{seeded} graph matching, and vertices with a known correspondence are referred as \emph{seeds}. The methods we developed can incorporate seed vertices by adapting the Frank-Wolfe methodology of \texttt{FAQ} to incorporate seeds as in  \citep{JMLR:v15:lyzinski14a,Fishkind2019}. To formulate this extension, without loss of generality, assume that the first $n_1<n$vertices  are seeds, and let $R\in\Pi_{n_1}$ be the known corresponding permutation that unshuffles these  vertices. The seeded graph matching problem can then be formulated in an analogous way as problem~\eqref{eq:GMP-uni-bip} with the additional constraint that $ P = (R\oplus Q)\in \Pi_n$, with $Q\in\Pi_{n-n_1}$ and $\oplus$ is the direct sum of matrices, such that $P_{ij}=R_{ij}$ for $i,j\in[n_1]$, $P_{ij}=Q_{i-n_1 + j-n_1}$ for $i,j\in[n]\setminus[n_1]$, and $P_{ij}=0$ otherwise. This restriction only requires to modify the update of $\hat{D}^{(t+1,\lambda)}$ in Algorithms~\ref{alg:Bip-inv} and \ref{alg:Bip-pseudo} via Equation~\eqref{eq:step-faq}. Thus, in order to find an estimator under this new constrain, we introduce an intermediate step to estimate the part of the unshuffling permutation that is unknown, and update $\hat{D}^{(t, \lambda)}$ accordingly as
\begin{align}
\hat{J}^{(t,\lambda)} & = \argmin_{J\in\mathcal{D}_{n-n_1}} \Tr\left((R\oplus J)^\top A(R\oplus J)|\hat{\Theta}^{(t,\lambda)}|\right), \label{eq:step-seeded}\\
\hat{D}^{(t,\lambda)} & = R \oplus \hat{J}^{(t,\lambda)}.
\end{align}
The resulting optimization problem is very similar to Equation~\eqref{eq:step-faq}, and the Frank-Wolfe methodology can be used analogously to approximately solve \eqref{eq:step-seeded} via the \texttt{SGM} algorithm  \citep{Fishkind2019},  implemented in the \texttt{IGraphMatch} package.

%%%%%%%%%%%%%%%%%%%%%%%%%%%%%%%%%%%%%%%%%%%%%%%%%%%%%%%%%%%%%%%%%%%%%%%%%%%%%%%%%%%%%%%%%%%%%%%%
\section{Simulations\label{sec:simulations}}
The accuracy of the proposed methods is evaluated in synthetic data. Given a graph $A$ (with a distribution that will be specified), the columns of the bipartite graph $B$ follow the Ising model distribution \eqref{eq:isingmodel}. The number of common nodes  in the graphs is set to $n=100$ while varying $m$, the nodes that are only in the bipartite graph. The value of the parameters is set to $\Theta_{ij}^\ast=0.4$ for all $(i,j)$ for which $A_{ij}=1$, and $\mu_i=-1/2\sum_{j=1}^n\Theta_{ij}^\ast$ for $i\in[n]$ to make the marginal distributions of the rows of $B$ approximately the same. The graph $A$ is generated using two different models:
\begin{itemize}
    \item\textbf{Chain graph:} Given the ordering of the vertices according to the rows and columns of the adjacency matrix, each vertex is connected by an edge to its adjacent neighbor, such that $A_{ij}=A_{ji}=1$ for $i\in[n-1]$ and $j=i+1$, 
    \item\textbf{Erd\H{o}s-R\'enyi (ER) graph:} The edges of $A$ are independent Bernoulli random variables with $\p(A_{ij}=1)=0.05$ for $i>j$.
\end{itemize}
In the first scenario with a chain graph, the simple dependency structure usually facilitates the estimation of the edges in the graphical model, but the regularity of the structure makes vertex matching a hard task since only a few errors in estimating the edges of $A$ can destroy the identifiability of the vertex order. On the contrary, the long-range correlations in the ER graph make graphical model estimation a harder problem \citep{Bento2009}, but vertex matching is an easier task due to the irregular structure of $A$ \citep{Arroyo2018}. 

In addition to the bipartite matching methods introduced in Algorithms~\ref{alg:Bip-inv} (B-InvCov) and \ref{alg:Bip-pseudo} (B-Pseudo), we measure the performance of unipartite matching between $A$ and a collapsing of $B$ to an $n\times n$ matrix. 
We consider the one-mode projection $\widetilde{B}$ (C-OMP), the sample covariance matrix $\hat{\Sigma}$ of the rows of $B$ (C-Cov), and estimates of the edges of the graphical model for $B$ based on the graphical lasso \citep{Friedman2008a} (C-GLasso) and the method of \citep{Meinshausen2006} (C-M\&B). %\cite{}
To implement the last two estimators, we use the R package \texttt{huge}  \citep{Zhao2012a}, and the parameter tuning is done with the default options of this package. It is important to remark that many alternative methods exist, especially for the problem of learning a graph dependency structure \citep{she2019indirect}, but we restrict to \citep{Meinshausen2006}  and  \citep{Friedman2008a} as these are common approaches for the Ising models.
Given a bipartite graph collapsing, graph matching between $A$ and the collapsed graph is performed with the FAQ algorithm \citep{FAQ}. In Algorithms~\ref{alg:Bip-inv} and \ref{alg:Bip-pseudo}, we set the maximum number of iterations $T^\ast$ as $20$,  and the grid for the tuning parameter $\lambda$ contains $S^\ast=10$ values logarithmically spaced between $10^{-2.5}$ and $10^{-0.5}$. An implementation of our methods can be found at \url{https://github.com/jesusdaniel/rBipartiteUnipartiteMatch}.

The methods result in a estimator $\hat{P}\in\Pi_n$ of the true unshuffling permutation $I_n$. We measure the vertex matching error as $\frac{1}{2n}\|\hat{P}- I_n\|_F^2$, which counts the percentage of vertices incorrectly matched, and the edge matching error as the percentage of incorrect edges $\frac{1}{2\|A\|_F^2}\|A - \hat{P}^\top A\hat{P}\|_F^2$. In addition, some methods provide an estimate $\hat{W}$ of the edges of the undirected graphical model $W=P^\top AP$. For Algorithms \ref{alg:Bip-inv} and \ref{alg:Bip-pseudo} this is given by the non-zero pattern of the estimated $\hat{\Theta}_{\hat{P}}$. The accuracy of $\hat{W}$ is measured using the edge false positive and false negative rates, defined as
\begin{align*}
\text{FPR} & = \frac{ \sum_{i>j}\mathbbm{1}(\hat{W}_{ij}=1)\mathbbm{1}({W}_{ij}=0)}{\sum_{i>j} \mathbbm{1}({W}_{ij}=0)} ,\\
\text{FNR} & =\frac{ \sum_{i>j}\mathbbm{1}(\hat{W}_{ij}=0)\mathbbm{1}({W}_{ij}=1)}{\sum_{i>j} \mathbbm{1}({W}_{ij}=1)}.    
\end{align*}
These errors are measured on 30 different Monte Carlo simulations of each scenario for different values of $m$. Running an instance of the simulations with $n=100$, $m=100$, $|S^\ast|=10$ and $T^\ast = 20$ takes about 4.16 minutes for the penalized inverse covariance estimation method (Algorithm~\ref{alg:Bip-inv}), and 13.88 minutes for the pseudolikelihood method (Algorithm~\ref{alg:Bip-pseudo}) on a 2.9 GHz computer. The computational time can be modified by changing the resolution of the grid of tuning parameters via $S^\ast$ or the maximum number of iterations $T^\ast$ at the expense of changes in the accuracy.

The average matching and estimation errors across all the Monte Carlo simulations are shown in Figures~\ref{fig:sim-graph_matching} and \ref{fig:sim-graph_estimation}. In general, methods improve their performance as $m$ grows, both for matching and graphical model estimation, although collapsing methods based on covariance or one-mode projection do not show a noticeable improvement. Our bipartite methods have a superior performance in graph matching with respect to collapsed matching methods. Among those, only the methods based in graphical model estimation show a decent performance for large $m$, when the edge set of the graphical model is accurately estimated as shown in Figure~\ref{fig:sim-graph_estimation}. The results of this figure also suggest by the performance of our methods that the knowledge of $A$, up to some permutation, does reduce the graphical model learning error in the ER graph scenario, where this problem is harder. The FPR is low in our methods because they are limited to discovering at most the same number of edges as $A$, and this graph is sparse in our simulation. For the other methods, the FPR and FNR  depend on the choice of the penalty parameter as usual, and the low values of the FPR are a consequence of the model selection procedure that seems to prefer sparser graphs.

\begin{figure*}[ht!]
    \centering
    \begin{subfigure}[b]{0.8\textwidth}
        \includegraphics[width=\textwidth]{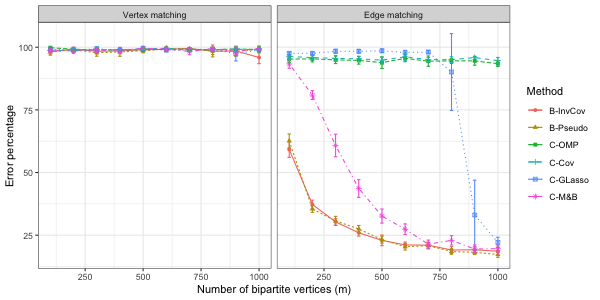}
        \caption{Chain graph.}
    \end{subfigure}
    \begin{subfigure}[b]{0.8\textwidth}
        \includegraphics[width=\textwidth]{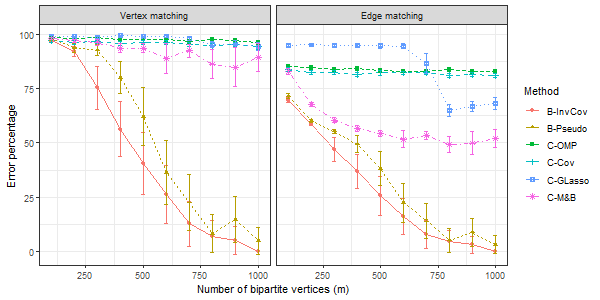}
        \caption{ER graph}
    \end{subfigure}
    \caption{Graph matching percentage error for different methods and unipartite graph models. The points are the average of 30 independent replications, with the bars representing the size of a two standard error interval.
    Our methods show superior performance than matching the collapsed bipartite graph.}
    \label{fig:sim-graph_matching}
\end{figure*}

\begin{figure*}[ht!]
    \centering
    \begin{subfigure}[b]{0.8\textwidth}
    \includegraphics[width=\textwidth]{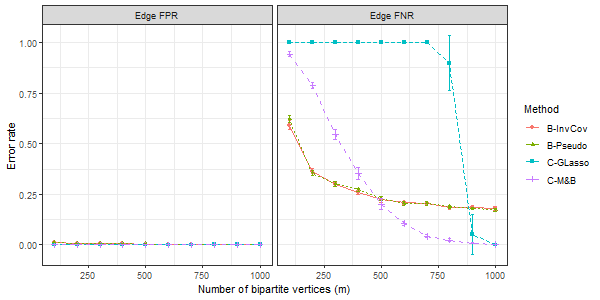}
    \caption{Chain graph.}
    \end{subfigure}
    \begin{subfigure}[b]{0.8\textwidth}
    \includegraphics[width=\textwidth]{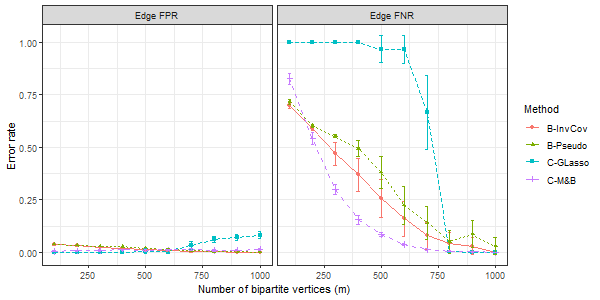}
    \caption{ER graph.}
    \end{subfigure}
    \caption{Error in estimating the graphical model of  the bipartite graph. The points are the average of 30 independent replications, with the bars representing the size of a two-standard error interval. All methods show a good control of false positives, but for small sample sizes our bipartite matching algorithms present better edge discovery due to the knowledge of $A$.}
    \label{fig:sim-graph_estimation}
\end{figure*}

%%%%%%%%%%%%%%%%%%%%%%%%%%%%%%%%%%%%%%%%%%%%%%%%%%%%%%%%%%%%%%%%%%%%%%%%%%%%%
\section{Evaluation on real networks}

We next explore the application of our methodology on two real-data network pairs.
\subsection{Co-authorship network and paper-authorship networks\label{sec:citation}}
The co-authorship and citation networks between statisticians
\citep{ji2016coauthorship} contain data from scientific papers in some of the most popular journals in statistics. 
We use these data to construct a co-authorship unipartite graph and a bipartite graph between authors and papers cited by them (Figure~\ref{fig:statsnet}), and then evaluate the performance of our approach in matching the authors when they are anonymized across the two graphs.

The original data consists of a unipartite directed graph of the citations between $m'=3248$ papers, and a bipartite graph that indicates  authorship of those papers by their corresponding $n'=3608$ authors. Based on these data, we constructed two graphs as follows:
\begin{itemize}
    \item\textbf{Co-authorship graph.} For every pair of authors in the list, two authors are connected if they have published a paper together, resulting in a matrix $A'\in\{0,1\}^{n'\times n'}$.
    \item\textbf{Author-citation network.} A matrix $B'\in\{0,1\}^{n'\times m'}$ encodes a bipartite network between authors and papers, in which each author is connected to the papers that the author has cited.
\end{itemize}
To select a substantial subset of vertices with enough information for performing graph matching, vertices with low or high degree were removed as follows. First, the authors that have cited at least 10 papers on the list were selected. Second, in the resulting subgraph of $A'$, authors that have degrees larger than 1 and smaller than 10 were chosen. Removing the vertices with larger degree step is performed to improve  numerical stability, because having too many unconstrained variables in Equation~\eqref{eq:GMP-uni-bip} can make the solution non-unique; alternatively, a small ridge or lasso penalty can be added to achieve the same goal. Next, the largest connected component of the resulting subgraph of $A'$ was chosen. This step is performed to improve the performance of all the methods, since small components are hard or even impossible to identify and match correctly, but in principle this step can be avoided. Finally, vertices for which rows in $B'$ are co-linear were discarded, again for numerical stability reasons.  After these steps, the resulting graphs $A$ and $B$ have $n=161$ shared vertices that represent the number of authors, and $m=1109$ vertices representing the papers.

The performance of different algorithms is compared in a similar way as in Section~\ref{sec:simulations}.
We include seeded vertices (on the $x$-axis, we have the fraction of randomly chosen seeds), 
and we perform seeded graph matching, measuring the vertex and edge matching errors in the subgraph induced by the unseeded vertices. 
For each fraction of seeds considered, the process of selecting random vertices as seeds is repeated 10 times and results are averaged over these 10 iterates; results are summarized in Figure~\ref{fig:coauth}.
Most methods perform similarly well in terms of vertex matching, and overall the performance improves as more seeds are present, which is expected. 
While exact vertex matching is hard due to the presence of low degree vertices, our methods show superior performance in terms of edge matching, especially when the number of seeds is not very large. 
Note that the C-M$\&$B method of \citep{Meinshausen2006} %\cite{}
usually results in a very sparse estimated graph due to the difficulty of parameter selection, which explains the poor performance.

\begin{figure*}[ht!]
    \centering
	\includegraphics[width=0.8\textwidth]{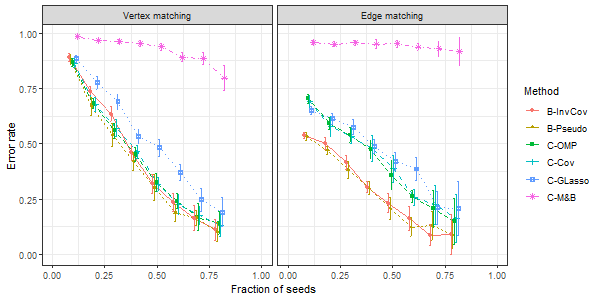}
	\label{fig:stats_match_err}
	\caption{Graph matching error in the co-authorship graph \citep{ji2016coauthorship} as a function of the fraction of seeds. Exact vertex matching is hard due to the small degree of the vertices, but our algorithms show a superior performance in edge matching.}
	\label{fig:coauth}
\end{figure*}

%%%%%%%%%%%%%%%%%%%%%%%%%%%%%%%%%%%%%%%%%%%%%%%%%%%%%%%%%%%%%%%
\subsection{MRI data}
Functional activation patterns in the brain have been linked to the structural connectivity,  which constraints  co-activation patterns in different brain regions \citep{abdelnour2014network,meier2016mapping,ng2012novel}. 
Assuming the existence of such relation, we use our matching algorithm to identify the mapping between brain regions in the structural connectivity and the functional activation time series data. 

The BNU1 data \citep{Zuo2014} contains  magnetic resonance imaging (MRI) brain data from 57 healthy subjects that were scanned twice within a period of 6 weeks. A post-processed version of these data was obtained from \url{https://neurodata.io/mri/}, computed with NeuroData's MR to graphs package (ndmg) \citep{Kiar2018}.
Each  scan session (corresponding to one of the two sessions from a particular subject, with 106 scan sessions available in total) produced a unipartite brain network derived from diffusion MRI (dMRI) data, 
and a data matrix containing functional MRI (fMRI)  measurements of blood oxygenation level-dependent (BOLD) signals  on $m=200$ time steps. For both data modalities, we employ the DS00278 parcellation, composed of $277$ brain regions, but in the experiments we focus on the subset of $n=253$ regions that are present in both dMRI and fMRI data. Each dMRI network is encoded into a $253\times 253$ binary matrix representing the existence of nerve tracts between brain regions, and a  data matrix of size $253\times 200$ represents the centered and standardized fMRI BOLD measurements of a particular scan session.  Figures~\ref{fig:dmri} and \ref{fig:fmri-net} show one of these pairs of data.

\begin{figure*}[ht!]
\centering
    \begin{subfigure}[t]{0.4\textwidth}
    \centering
        \includegraphics[width=\textwidth]{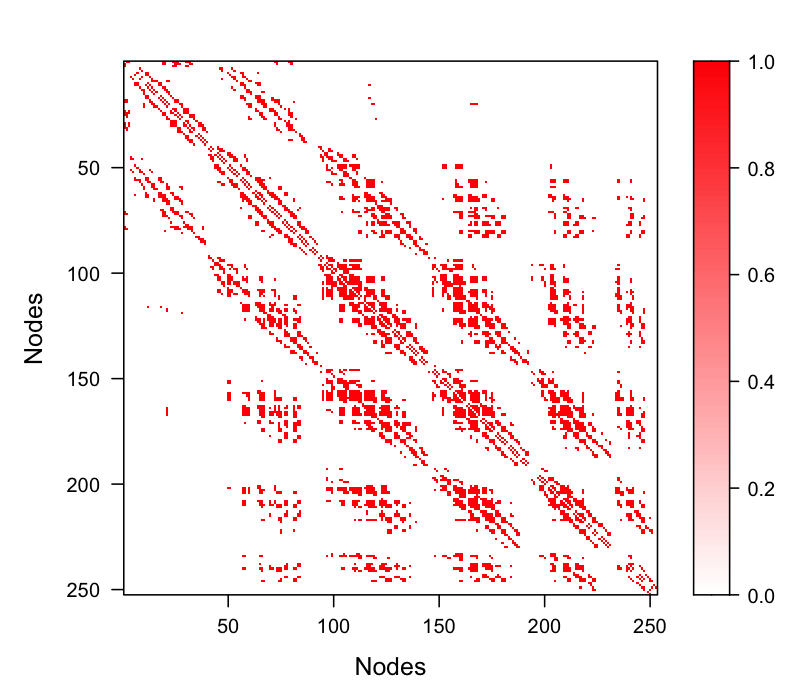}
        \label{fig:subfig-dMRI-net}
        \caption{Diffusion MRI brain network.}\label{fig:dmri}
    \end{subfigure}
    \begin{subfigure}[t]{0.58\textwidth}
    \centering
        \includegraphics[width=\textwidth]{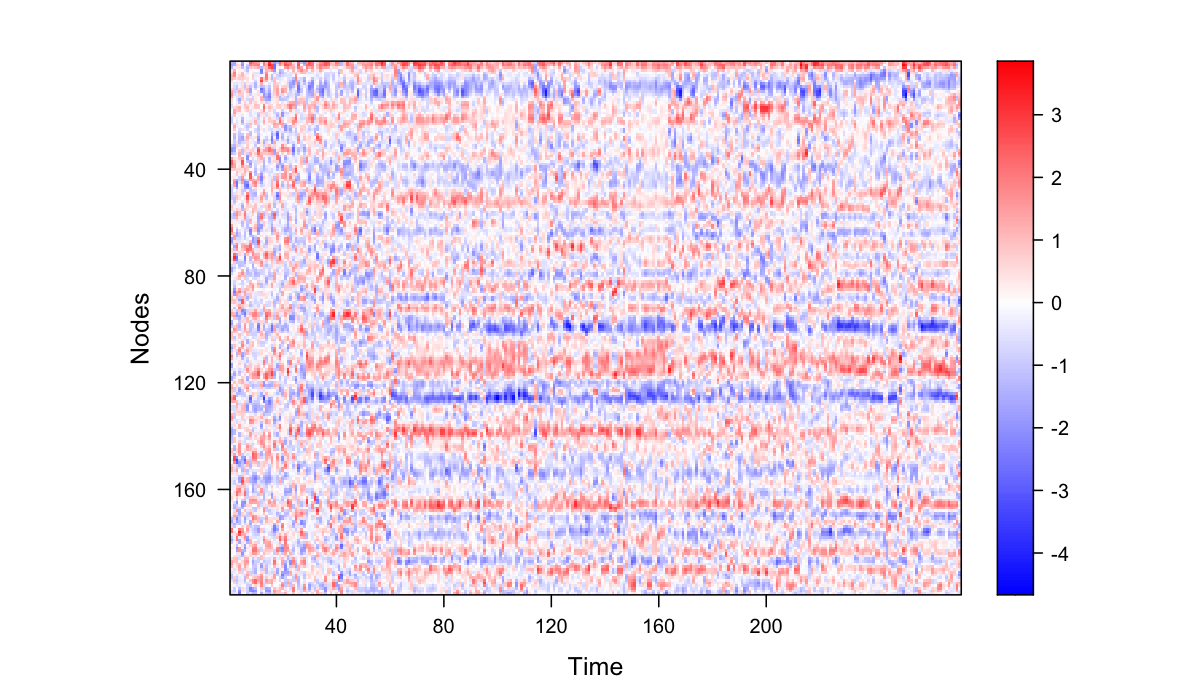}
        \caption{Functional MRI measurements.\label{fig:fmri-net}}
    \end{subfigure}
    \begin{subfigure}[c]{0.48\textwidth}
    \centering
        \includegraphics[width=0.8\textwidth]{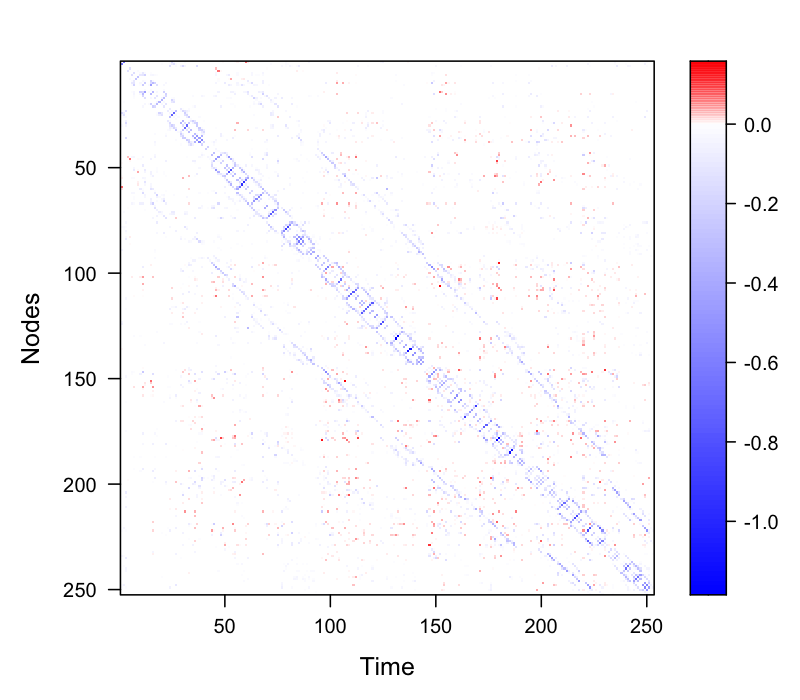}
        \caption{Sparse estimate of the fMRI inverse covariance matrix based on the graphical lasso.\label{fig:fmri-glasso}}
    \end{subfigure}
    \begin{subfigure}[c]{0.48\textwidth}
    \centering
        \includegraphics[width=0.8\textwidth]{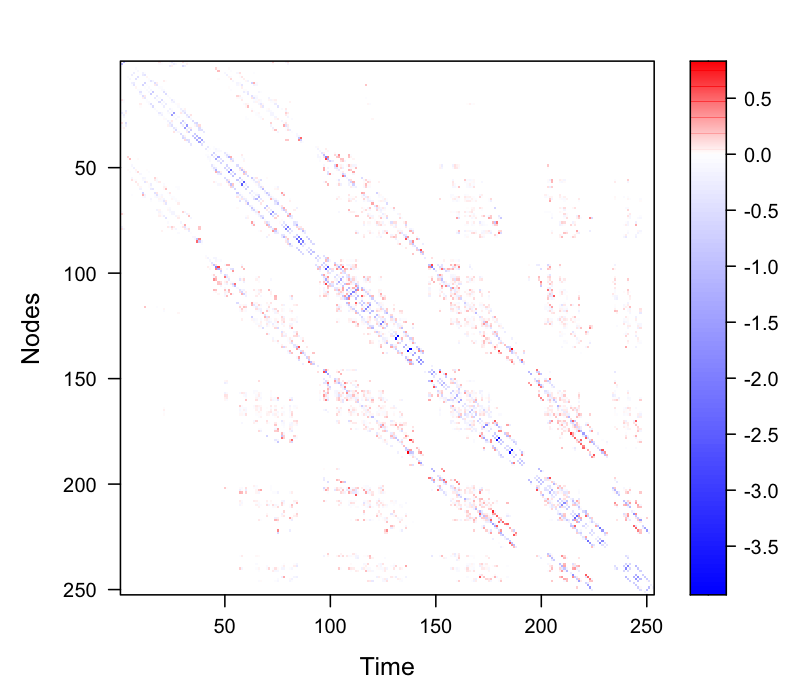}
        \caption{Inverse covariance matrix estimate constrained to the non-zeros of the dMRI network.}\label{fig:fmri-match}
    \end{subfigure}
	\label{fig:mri-nets}
	\caption{The top panels show the structural brain network (left) and the functional MRI time measurements after standardization (right) of one of the scan sessions from the BNU1 data. The bottom panels show estimates of the inverse covariance matrix between brain regions obtained from the whole set of fMRI data measurements. }\label{fig:dmrifmri}
\end{figure*}

The graphical lasso is often used to infer the functional connectivity between  brain regions \citep{ng2012novel,narayan2015sample} despite the dependency across time in the fMRI measurements. 
We use this method to obtain a sparse  estimate of the inverse covariance matrix of the fMRI data taking  all the 106 scan sessions simultaneously to improve the estimation (resulting in a matrix with 21200 time measurements in total). The matrix obtained with the \texttt{huge} package is shown in Figure~\ref{fig:fmri-glasso}, with the rows and columns of the matrix ordered in the same way as the dMRI graph. Many of the entries with the strongest signals in the graphical lasso estimate  correspond with edges in the adjacency matrix of the dMRI network (Figure~\ref{fig:dmri}). On the other hand, some of the  small weights in the graphical lasso estimate corresponding to zeros in the dMRI graph might not be significant, and thus removing them might be appropriate. Indeed, this shared structure across the data modalities has been exploited to infer functional connectivity in previous work (see for example \citep{ng2012novel}), and our methodology uses this correspondence to match the vertices. Figure~\ref{fig:fmri-match} shows the constrained maximum likelihood estimate of the fMRI inverse covariance matrix subject to the entries enforced to be zero for non-edges in the structural dMRI network. This matrix corresponds to the solution of our graph matching optimization problem~\eqref{eq:GMP-uni-bip} under a Gaussian graphical model when the vertices are correctly matched. In comparison with the graphical lasso solution, many edges are removed, resulting in inflated weights on the remaining edges.

We now investigate the performance of different methods in aligning the vertices across the fMRI and dMRI data when the labels of some of the vertices are unknown. This task is often required in order to analyze different data modalities jointly
\citep{saad2009new}. For this goal, we treat the fMRI data matrix as a weighted bipartite network between brain regions and time points, and use some of the methods discussed in the previous sections to recover the correct vertex alignment. We also study the effect of the number of observations in the fMRI data by concatenating the data from multiple scans. 

\begin{figure*}[ht!]
    \centering
	\includegraphics[width=0.9\textwidth]{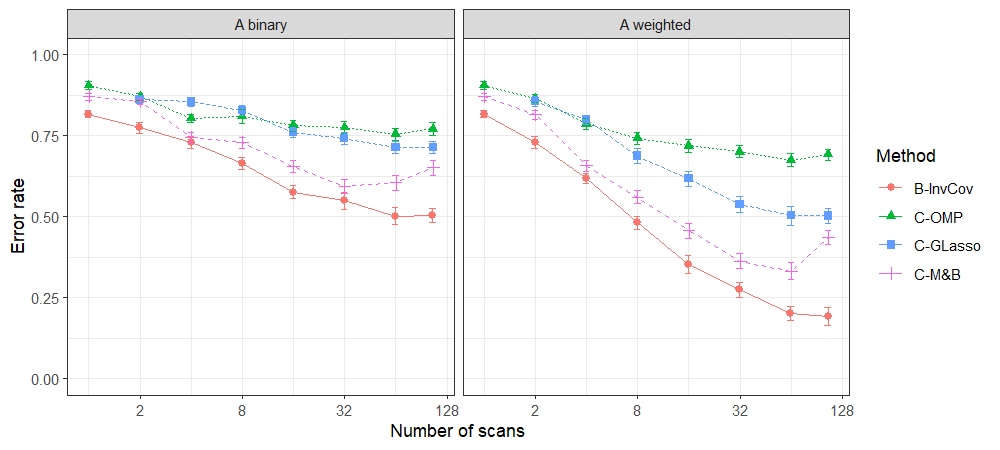}
	\caption{Graph matching error between the dMRI structural brain graphs (left: a single dMRI graph; right: average of a given number of dMRI graphs) and the fMRI time series data as a function of the number of scans used. In all cases, $n/2$ vertices were randomly selected as seeds, and the average of 30 different seed selections (with 2 standard error bars) is reported.}
    \label{fig:MRI-gm-error}
\end{figure*}

Figure~\ref{fig:MRI-gm-error} shows the average fraction of vertices that were incorrectly matched by different algorithms as a function of the number of scan sessions used. The left panel shows the result of matching a single binary dMRI network of one of the samples in the data, and the concatenated fMRI measurements for a given number of scan sessions. In addition, we also study the performance of a method based on a weighted unipartite network constructed from averaging the dMRI graphs of a given number of samples (right panel); we note that although our methods were designed for binary networks, Equations~\eqref{eq:step-graphicallasso} and~\eqref{eq:step-faq} can also handle weighted edges in $[0,1]$. As we did in Section~\ref{sec:citation}, we randomly select a fraction of the vertices as seeds for the matching algorithms (50\% in our experiments), but we observe that the results are qualitatively similar for different sizes of the seeds set. The results show that our graph matching method based on sparse inverse covariance estimation (Algorithm~\ref{alg:Bip-inv}) has the best average performance in all cases, followed by the unipartite matching using the estimated fMRI graph based on \citep{Meinshausen2006} (C-M\&B). These  results thus support the idea that the assumption of a shared edge set across the two data modalities is useful from the graph matching perspective.
Most of the methods significantly improve their performance as the number of scans increase, and a large sample size is required to achieve an accurate solution. However, this is not the case for the method based on  \citep{Meinshausen2006}, which uses penalized node-wise regression to estimate the edges. A possible explanation is that the density of the network estimated by C-M\&B increases with the number of observations, and thus the level of sparsity might not be appropriate for graph matching purposes. Finally, we also observe that the methods that use a weighted dMRI network from multiple measurements generally perform better, probably due to a better estimation of the structural connectivity, which highlights the need of large amounts of MRI data to obtain accurate matching.

\section{Conclusion}
Network data is ubiquitous.
While multiple methodologies have been developed to analyze this type of data, most of the studies have focused on the unipartite graph setting. 
More complex data structures often cannot be fully represented with unipartite graphs, and in practice, information is commonly discarded or collapsed in order to adapt the data to the existing methodologies. 
In this paper we have shown that from the graph matching perspective, collapsing the graph into a unipartite network is not always the right approach and can often fail. However, addressing the problem with methods tailored for the specific data yield significant gains in accuracy. 
More complex data designs, such as vertex or edge attributes, missing data, multilayer networks, among others, present new challenges in network analysis, and we are working to adapt our approach to these more complex settings.

%% file: appendix-content.tex
We start with some preliminaries and introduce the following symbols to simplify notation. For a matrix $R\in\real^{n\times n}$ and a set of indexes $\mathcal{I}\subset [n]\times[n]$, define $\|R\|_{F, \mathcal{I}}$ as the Frobenius norm of the matrix $R$ calculated only on the entries with indexes in $\mathcal{I}$, that is,
$$\|R\|_{F, \mathcal{I}}^2 := \sum_{(i,j)\in\mathcal{I}}R_{ij}^2.$$
Similarly, for a pair of matrices $R,S\in\real^{n\times n}$, we define the inner product of the entries with indexes in $\mathcal{I}$ as
$$\left\langle R, S\right\rangle_{\mathcal{I}} := \sum_{(i,j)\in\mathcal{I}} R_{ij}S_{ij}.$$
Given a permutation matrix $P\in\Pi_n$, we also define $\mathcal{I}_P$ as the set of indexes for which the solution for $\Theta$ to the constrained optimization graph matching problem with fixed $P$ can take non-zero values, that is,
$$\mathcal{I}_P:= \{(i,j)\in[n]\times [n]: i=j \text{ or } (P^\top AP)_{ij}=1)\}. $$
We also define $\mathbb{M}_P\subset\real^{n\times n}$ as the set of positive definite matrices that only take non-zero values in the entries corresponding to the set defined above, that is,
$$\mathbb{M}_P :=\{\Theta\in\real^{n\times n}:\Theta\succ 0\text{ and }\Theta_{ij}=0\text{ if }(i,j)\notin \mathcal{I}_P\}.$$
Observe that if $R\in\mathbb{M}_P$, then $\left\langle R, S\right\rangle = \left\langle R, S\right\rangle_{\mathcal{I}_P}$.

 For real matrices $R,S\in\real^{n\times n}$, we write $\ve{R}\in\real^{n^2}$ as the vectorized adjacency matrix, and $R\otimes S\in\real^{n^2\times n^2}$ as the Kronecker product of $R$ and $S$.
 
We also recall some properties of the log-determinant function. For a pair of positive definite matrices $\Theta_1, \Theta_2\in\real^{n\times n}$, a Taylor expansion of the function for $\Theta_2$ around $\Theta_1$ is given by
    \begin{align}
        \log|\Theta_2| & = \log |\Theta_1|  + \left\langle\Delta, \Theta_1^{-1}\right\rangle - \frac{1}{2}\ve{\Delta}^\top(\Theta_1 + v\Delta)^{-1}\otimes(\Theta_1 + v\Delta)^{-1}\ve{\Delta}, \label{eq:taylor-logdet}
    \end{align}
   for some $v\in[0,1]$, and with $\Delta = \Theta_2 - \Theta_1$. It is also useful to bound the last term in the previous equation. Using the fact that $\gamma_{\max}(R\otimes R) = \gamma_{\max}^2(R)$ and $\gamma_{\min}(R\otimes R) = \gamma_{\min}^2(R)$ for a positive semidefinite matrix $R$, we have
\begin{align}
    \ve{\Delta}^\top(\Theta + v\Delta)^{-1}\otimes(\Theta + v\Delta)^{-1}\ve{\Delta} & \leq \frac{\|\Delta\|^2_F}{\gamma_{\min}^2(\Theta + v\Delta)} \label{eq:kroneckerupperbound}\\
    \ve{\Delta}^\top(\Theta + v\Delta)^{-1}\otimes(\Theta + v\Delta)^{-1}\ve{\Delta} & \geq \frac{\|\Delta\|^2_F}{\gamma_{\max}^2(\Theta + v\Delta)} \label{eq:kronecker-lowerbound}
\end{align}

\section{Proof of Theorem 1, part 1.}
\begin{proof}[Proof of Theorem~\ref{thm:main}, part 1.]
First, assume that 
 $\beta=0$ in the model \eqref{eq:GLM-full}. Then, the empirical log-likelihood for a value of $m$ is given by
 \begin{equation}
    \hat{\ell}_{m}(\Theta) := \frac{1}{m}\sum_{k=1}^m\left(\sum_{i\neq j}\Theta_{ij}B_{ik}B_{jk} -2\sum_{i=1}^n\Theta_{ii}C(B_{ik})\right) -\log(Z(\Theta)),\label{eq:empirical-loglik}
 \end{equation}
with $Z(\Theta)$ the partition function. Note that the log-likelihood function here is concave because the model belongs to the exponential family. 

Define $\hat{\Theta}_{m} := \argmax_{\Theta}\hat{\ell}(\Theta)$, and 
$\hat{\Theta}_{P,m} := \argmax_{\Theta}\hat{\ell}(\Theta)$ subject to $\Theta_{ij}(1-(P^\top AP)_{ij})=0$. To show that log-likelihood function is well behaved   in order to show consistency of $\hat{\Theta}_{P^\ast,m}$ and $\hat{\Theta}_{m}$ (see for example 5.2.2 of \cite{bickel2015mathematical}), we observe that the normalizing constant in Equation~\eqref{eq:GLM-full} is finite for any $\Theta\in\real^{n\times n}$ for both the Ising model and the multivariate Gaussian distribution, so the natural parameter space is open.   
In addition, the restriction in the optimization problem enforces the structure of $\hat{\Theta}_{P^\ast, m}$, to have zero entries in the same places as $\Theta^\ast$, so the effective parameters are the non-zero entries of $\Theta^\ast$.
Therefore, the estimators $\hat{\Theta}_{P^\ast,m}$ and $\hat{\Theta}_{m}$ are consistent for $\Theta^\ast$. By the weak law of large numbers, $\frac{1}{m}\sum_{k=1}^m B_{ik}B_{jk}\rightarrow \e[B_{i1}B_{j1}]$, and hence, by
continuity of the function $Z(\Theta)$, $\hat{\ell}_m(\hat{\Theta}_m)\overset{\mathbb{P}}{\rightarrow}\ell^\ast(\Theta^\ast)$ and $\hat{\ell}_m(\hat{\Theta}_{P^\ast,m})\overset{\mathbb{P}}{\rightarrow}\ell^\ast(\Theta^\ast)$. 

Now, suppose that there exists sequences  $\{m_k\}_{k=1}^\infty\subset\mathbb{N}$ and  $\{P_{m_k}\}_{k=1}^\infty$ such that for each $m_k$,  the solution of \eqref{eq:GMP-uni-bip} for $\hat{\ell}_{m_k}$ is $P_{m_k}$. Because $\Pi_n$ is finite, there exists a permutation $\tilde{P}$ for which $\tilde{P}=P_{m_k}$ for infinitely many values of $k$. This permutation  should satisfy 
$$\hat\ell_{m_k}(\hat{\Theta}_{m_k}) \geq \hat\ell_{m_k}(\hat{\Theta}_{\tilde{P}, m_k}) \geq \hat\ell_{m_k}(\hat{\Theta}_{P^\ast, m_k}).$$
The left inequality is a consequence of the optimality of $\hat{\Theta}_{m_k}$ in the unrestricted problem, while the right inequality is the assumption in the definition of $\tilde{P}$.
Hence, $\hat\ell_{m_k}(\hat{\Theta}_{\tilde{P}, m_k})\overset{\mathbb{P}}{\rightarrow}{\ell^\ast(\Theta^\ast)}$. The sequence of functions $\{\hat\ell_m\}$ and $\ell^\ast$ are concave (Theorem 1.6.3 of \cite{bickel2015mathematical}, in addition to the linearity w.r.t. $\Theta$ in Equation~\eqref{eq:empirical-loglik}). This
implies that for any compact set $\mathcal{K}\subset \real^{n\times n}$, $\hat{\ell}_m$ converges uniformly in probability to $\ell^\ast$ in $\mathcal{K}$ (Theorem II.1 of \cite{andersen1982cox}). Taking $\mathcal{K}_{\Theta^\ast}$ such that $\Theta^\ast$ is in the interior of $\mathcal{K}_{\Theta^\ast}$, condition (i) implies that $\lim\mathbb{P}(\hat{\Theta}_{\tilde{P},m}\in\text{int}(\mathcal{K}_{\Theta^\ast}))=1$, and therefore the uniform convergence implies that $\hat{\Theta}_{\tilde{P},m}\overset{\mathbb{P}}{\rightarrow}\Theta^\ast$. 
Since $\Theta^\ast_{ij}(1(\tilde{P}^\top A\tilde{P})_{ij})=0$, condition (ii) guarantees that $\tilde{P}^\top A\tilde{P} = (P^\ast)^\top AP^\ast$. This is valid for any $P$ that appears infinitely many times on $\{P_{m_k}\}$, which implies the result. 

The last statement of the theorem trivially holds because if $A$ does not have non-identity automorphisms then $\tilde{P}=P^\ast$ is the only solution for $P$ in $\tilde{P}^\top A\tilde{P} = (P^\ast)^\top AP^\ast$. The proof for the case in which $\beta\neq 0$ follows mutatis mutandis.
\end{proof}

\section{Proof of Theorem 1, part 2}

%%%%%%%%%%%%%%%%%%%%%%%%%%%%%%%%%%%%%%%%%%%%%%%%%%%%%%%%%%%%%%%%%%
The next lemma follows from an application of Lemma A.3 of \cite{bickel2008regularized} for the concentration of the entries of  the sample covariance matrix. In this section, we use $\Sigma:=(\Theta^\ast)^{-1}$ to denote the covariance matrix of the Gaussian graphical model for $B^{(m)}$, and $\widehat\Sigma = B^{(m)}(B^{(m)})^\top$ is the sample covariance matrix.

\begin{lemma} \label{lemmma:covariance_frobenius}
Suppose that $\gamma_{\max}(\Sigma)\leq \kappa_1$ for some $\kappa_1>0$. Then, there exist positive constants  $c_1, c_2$ and $\delta$  that only depend on $\kappa_1$  such that 
%\begin{equation*}
%    \mathbb{P}\left(\|\widehat\Sigma - \Sigma\|_{F,\mathcal{I}}^2 >  c_1\frac{|\mathcal{I}|\log(|\mathcal{I}|)}{m} \right) \leq \frac{c_2}{|\mathcal{I}|}.
%\end{equation*}

%\begin{equation*}
%    \mathbb{P}\left(\|\widehat\Sigma - \Sigma\|_{F,\mathcal{I}}^2 >  t_1 \right) \leq c_1\exp\left(\log|\mathcal{I}| - c_2\frac{mt}{|\mathcal{I}}\right).
%\end{equation*}

\begin{equation*}
    \mathbb{P}\left(\max_{P\in\Pi_n}\|\widehat\Sigma - \Sigma\|_{F,\mathcal{I}_P}^2 >  t \right) \leq c_1\exp\left(2\log n - c_2\frac{mt}{\|A\|_F^2 + n}\right).
\end{equation*}
for any  $t\leq \delta (\|A\|_F^2 + n)$.
\end{lemma}

\begin{proof}
Observe that, according to \cite{bickel2008regularized},
\begin{align*}
    \mathbb{P}\left(\max_{i,j} (\widehat\Sigma_{ij} - \Sigma_{ij})^2 > t\right) \leq & \mathbb{P}\left(\bigcup_{i,j}(\widehat\Sigma_{ij} - \Sigma_{ij})^2 > t\right)\\
    \leq & \sum_{i=1}^n\sum_{j=1}^n
     \mathbb{P}\left((\widehat\Sigma_{ij} - \Sigma_{ij})^2 > t\right)\\
     \leq & c_1n^2\exp\left( -c_2mt\right),
\end{align*}
for any $t\leq \delta$, with $c_1, c_2$ and $\delta$ some positive constants that only depend on $\kappa_1$. Note that $|\mathcal{I}_P|=\|A\|_F^2+n$. Hence,
\begin{align*}
    \mathbb{P}\left(\max_{P\in\Pi_n}\|\widehat\Sigma - \Sigma\|_{F,\mathcal{I}_P}^2 > t\right)
    \leq & \mathbb{P}\left(\max_{P\in\Pi_n}|\mathcal{I}_P|\max_{i,j}(\widehat\Sigma_{ij} - \Sigma_{ij})^2 > t\right)\\
    = & \mathbb{P}\left(\max_{i,j}(\widehat\Sigma_{ij} - \Sigma_{ij})^2 >\frac{t}{\|A\|_F^2+n} \right)\\
    \leq & c_1\exp\left(2\log n - c_2\frac{mt}{\|A\|_F^2 + n}\right),
\end{align*}
for any $t\leq \delta (\|A\|_F^2+n)$.

%where the last bound comes from Lemma A.3 in \cite{bickel2008regularized} and  holds for any positive $t$ such that $\sqrt{\frac{t}{|\mathcal{I}|}}\leq \delta'$, with $\delta', c_1'$ and $c_2'$ some positive constants that only depend on $\kappa_1$.
%Taking $t =\frac{2|\mathcal{I}|\log|\mathcal{I}|}{c_1'm}$, $c_1 = 8/c_1'$, $c_2=c_2'$ and $\delta = c_1'(d')^2/2$ we obtain the result.
\end{proof}
%%%%%%%%%%%%%%%%%%%%%%%%%%%%%%%%%%%%%%%%%%%%%%%%%%%%%%%%%%%%%%%%%%

The next lemma studies the concentration of the estimator of the graphical model parameters restricted to a  given permutation $P\in\Pi_n$, that is
\begin{equation*}
    \widehat\Theta_P = \argmax_{\Theta} \hat{\ell}(\Theta) \text{ subject to }\Theta_{ij}(1-(P^\top A P)_{ij}) = 0.
\end{equation*}
To prove this result, we follow a similar argument to \cite{Rothman2008}.

\begin{lemma} \label{lemma:theta-error} Suppose that $\gamma_{\max}(\Sigma)\leq \kappa_1$ and $\max_{P\in\Pi_n}\gamma_{\max}(\widetilde\Theta_P) \leq \kappa_2$ for some positive $\kappa_1, 
\kappa_P$. Then, there exists positive constants $c_1, c_2$ and $\delta$ that only depend on $\kappa_1$ such that
\begin{equation*}
\mathbb{P}\left(\max_{P\in\Pi_n}\|\widehat\Theta_P - \widetilde\Theta_P\|_F^2 > t\right) \leq  c_1\exp\left(2\log n - c_2\frac{mt}{8\kappa_2(\|A\|_F^2 + n)}\right)    
\end{equation*}
for any $t\leq \min\left\{ \frac{1}{8\kappa_2}, 8\kappa_2\delta (\|A\|_F^2 +n) \right\}$.
\end{lemma}
%%%%%%%%%%%%%%%%%%%%%%%%%%%%%%%%%%%%%%%%%%%%%%
\begin{proof}
We show that for any $\Theta\in\mathbb{M}_P$ and some positive values $t_1$ and $t_2$ such that
$$t_1 < \|\Theta - \widetilde\Theta_P\|_F < t_2,$$
the  loglikelihood function satisfies $\hat{\ell}(\Theta) < \hat{\ell}(\widetilde\Theta_P) \leq \hat{\ell}(\widehat\Theta_P)$. Hence,  the concavity of $\hat{\ell}$ implies that the optimal solution $\widehat\Theta_P$ should satisfy $\|\widehat\Theta_P - \widetilde\Theta_P\|_F \leq t_1$. 

For a matrix $\Theta \in\mathbb{M}_P$, write $\Delta = \Theta - \widetilde\Theta_P$. The Taylor expansion of the log-determinant function in Equation~\eqref{eq:taylor-logdet} and the inequalities in Equation~\eqref{eq:kronecker-lowerbound}
imply that there exists a constant $v\in[0,1]$ such that the difference between the  values of the loglikelihood function at $\widetilde\Theta_P$ and $\Theta$ can be expressed as
\begin{align}
    \hat{\ell}(\widetilde\Theta_P) - \hat\ell(\Theta)  = & \left\langle \Delta, \widehat\Sigma\right\rangle - \left(\log|\Theta| - \log |\widetilde\Theta_P|\right)\nonumber \\
    = & \left\langle \Delta, \widehat\Sigma -\widetilde\Theta_P^{-1}\right\rangle_{\mathcal{I}_P} +\frac{1}{2}\ve{\Delta}^\top(\widetilde\Theta_P + v\Delta)^{-1}\otimes(\widetilde\Theta_P + v\Delta)^{-1}\ve{\Delta} \nonumber\\
    \geq & -\|\Delta\|_F\|\widehat\Sigma - \widetilde\Theta_P^{-1}\|_{F, \mathcal{I}_P} +\frac{\|\Delta\|_F^2}{2\gamma_{\max}^2(\widetilde\Theta_P + v\Delta)} \nonumber\\
    %= &\|\Delta\|_F\left( \frac{\|\Delta\|_F}{2\gamma_{\max}^2(\Theta_0 + v\Delta)} -\|\widehat\Sigma - \Sigma\|_{F, P}\right)\\
    \geq & \|\Delta\|_F\left( \frac{\|\Delta\|_F}{2(\gamma_{\max}(\widetilde\Theta_P) + \|\Delta\|_F)^2} -\|\widehat\Sigma - \widetilde\Theta_P^{-1}\|_{F, \mathcal{I}_P}\right).\label{eq:lemma-thetaalt-positive}
    %\geq & \|\Theta - \Theta_0\|_F\left(\frac{h}{2(\delta + h)^2} - \|\widehat\Sigma - \Sigma\|_{F, P} \right)\\
    %\geq & \|\Theta - \Theta_0\|_F\left(\frac{h}{8\delta^2} - \|\widehat\Sigma - \Sigma\|_{F, P} \right).
\end{align}
Notice that the last term in the previous equation can also be written as
$$\|\widehat\Sigma - \widetilde\Theta_P^{-1}\|_{F, \mathcal{I}_P} = \|\widehat\Sigma - \Sigma\|_{F, \mathcal{I}_P}.$$
To check this last equality, observe that the optimality conditions for $\widetilde\Theta_P$  imply 
$$\left[\nabla\hat\ell(\widetilde\Theta_P)\right]_{ij} = -\Sigma_{ij} + \left[\widetilde\Theta_P^{-1}\right]_{ij}=0,\quad\quad\text{for all }(i,j)\in\mathcal{I}_P.$$
Hence, a sufficient condition for Equation~\ref{eq:lemma-thetaalt-positive} to be strictly positive is that $\|\Delta\|_F < \gamma_{\max}(\widetilde\Theta_P)$ and
\begin{equation}
    8\gamma^2_{\max}(\widetilde\Theta_P)\|\widehat\Sigma - \Sigma\|_{F, \mathcal{I}_P}< \|\Delta\|_F , \label{eq:proof-thetaconc-positivecondition}
\end{equation}
and these inequalities have a solution if
\begin{equation*}
    \|\widehat\Sigma - \Sigma \|_{F, \mathcal{I}_P} < \frac{1}{8\gamma_{\max}(\widetilde\Theta_P)}.
\end{equation*}
Thus, conditioned on the event in the equation above, by the arguments stated at the beginning of the proof, Equation~\eqref{eq:proof-thetaconc-positivecondition} implies that
\begin{equation*}
    \|\widehat\Theta_P - \widetilde\Theta_P \|_F \leq  8\gamma^2_{\max}(\widetilde\Theta_P)\|\widehat\Sigma - \Sigma\|_{F, \mathcal{I}_P},
\end{equation*}
and hence, if
$\max_{P\in\Pi_n}\|\widehat\Theta_P - \widetilde\Theta_P\|_F > t$
then $\max_{P\in\Pi_n}8\gamma_{\max}^2(\widetilde\Theta_P)\|\widehat\Sigma - \Sigma\|_{F,\mathcal{I}_P}>t$ for any $t\leq \frac{1}{8\max_{P\in\Pi_n}\gamma_{\max}(\widetilde\Theta_P)}$. Therefore, using Lemma~\ref{lemmma:covariance_frobenius},
\begin{align*}
    \mathbb{P}\left(\max_{P\in \Pi_n}\|\widehat\Theta_P - \widetilde\Theta_P\|_F \geq t\right) & \leq \mathbb{P}\left(\max_{P\in \Pi_n} \|\widehat\Sigma - \Sigma\|_{F,\mathcal{I}_P} \geq \frac{t}{8\max_{P\in\Pi_n}\gamma_{\max}(\widetilde\Theta_P)}\right) \\
    & \leq c_2\exp\left(2\log n - c_1\frac{mt}{8(\|A\|_F^2 + n)\max_{P\in\Pi_n}\gamma_{\max}(\widetilde\Theta_P)}\right),
\end{align*}
for any $t\leq \min\left\{ \frac{1}{8\kappa_2}, 8\kappa_2\delta (\|A\|_F^2 +n) \right\}$.

%and the assumptions in the statement of Lemma~\ref{lemma:theta-error}, if $\delta_2 = \frac{1}{8c_1 }$, then

%\begin{equation*}
%    \mathbb{P}\left(\|\widehat\Sigma - \Sigma\|_{F,\mathcal{I}_P} \geq \frac{1}{8\|\widetilde\Theta_P\|}\right) \leq \mathbb{P}\left(\|\widehat\Sigma - \Sigma\|_{F,\mathcal{I}_P} \geq \frac{(\|A\|_F^2+n)\log n}{8\delta_2 m}\right) \leq \frac{c_2}{\|A\|_F^2 + n},
%\end{equation*}

%\begin{equation*}
%    \mathbb{P}\left(\|\widehat\Sigma - \Sigma\|_{F,\mathcal{I}_P} \geq \frac{t_1}{8\|\widetilde\Theta_P\|}\right) \leq \mathbb{P}\left(\|\widehat\Sigma - \Sigma\|_{F,\mathcal{I}_P} \geq \frac{(\|A\|_F^2+n)\log n}{8\delta_2 m}\right) \leq \frac{c_2}{\|A\|_F^2 + n},
%\end{equation*}

%where $c_1, c_2$ are the constants that appear on Lemma~\ref{lemmma:covariance_frobenius}. Therefore, 
%\begin{align*}
%    \mathbb{P}\left(\|\widehat\Theta_P - \widetilde\Theta_P\|_F \geq \frac{c_1(\|A\|_F^2+n)\log n}{m}\right) & \leq \mathbb{P}\left(\|\widehat\Sigma - \Sigma\|_{F,\mathcal{I}_P} \geq \frac{c_1(\|A\|_F^2+n)\log n}{ m}\right) \\
%    & \leq \frac{c_2}{\|A\|_F^2 + n}.
%\end{align*}

\end{proof}

Using the previous lemmas, we are now ready to prove the main result.

%%%%%%%%%%%%%%%%%%%%%%%%%%%%%%%%%%%%%%%%%%%%%%%%%%%%%%%%%%%%%%%%%%%%%%%%%
\begin{proof}[Proof of Theorem \ref{thm:main}, part 2.]
To prove that the correct permutation $P^\ast$ is exactly recovered, we show that with high probability,
$$\hat{\ell}(\widehat{\Theta}_{P^\ast}) > \hat{\ell}(\widehat{\Theta}_{P})\quad\quad\quad\text{for all }P\in\Pi_n\setminus\{P^\ast\}.$$

Define $\Delta_{P} := \widehat\Theta_{P}-\widetilde\Theta_{P}$,  $\Delta_{P^\ast} := \widetilde\Theta_{P^\ast}-\widehat\Theta_{P^\ast}$ and $\widetilde\Delta_P := \widetilde\Theta_P - \widetilde\Theta_{P^\ast}$. Then, using the Taylor expansion~\eqref{eq:taylor-logdet}, there exist $v_1,v_2\in[0,1]$ such that the difference of the loglikelihood values at the solutions $\widehat\Theta_{P^\ast}$ and $\widehat\Theta_P$
can be bounded from below as follows:
\begin{align}
\hat{\ell}(\widehat{\Theta}_{P^\ast}) - \hat{\ell}(\widehat{\Theta}_{P})  = & \left\langle\widehat\Theta_{P} - \widehat\Theta_{P^\ast}, \widehat\Sigma\right\rangle  -\left[\log|\widehat\Theta_{P}| - \log|\widehat\Theta_{P^\ast}|  \right]\nonumber\\
 = & \left\langle \widetilde\Theta_{P}-\widetilde\Theta_{P^\ast},\Sigma\right\rangle + \left\langle \widehat\Theta_{P}-\widetilde\Theta_{P},\Sigma\right\rangle + \left\langle \widetilde\Theta_{P^\ast}-\widehat\Theta_{P^\ast},\Sigma\right\rangle + \left\langle \widetilde\Delta_P,\widehat\Sigma-\Sigma\right\rangle +\nonumber\\
& \left\langle \widehat\Theta_{P}-\widetilde\Theta_{P},\widehat\Sigma - \Sigma\right\rangle   - \log|\widehat\Theta_P| + \left\langle \widetilde\Theta_{P^\ast}-\widehat\Theta_{P^\ast},\widehat\Sigma-\Sigma\right\rangle  + \log|\widehat\Theta_{P^\ast}| \nonumber\\
 = & \left\{ {\ell}^\ast(\widetilde{\Theta}_{P^\ast}) -{\ell}^\ast(\widetilde{\Theta}_{P})\right\} + \left\langle \Delta_{P},\widehat\Sigma - \Sigma\right\rangle + \left\langle \Delta_{P^\ast},\widehat\Sigma-\Sigma\right\rangle  + \left\langle \widetilde\Theta_{P}-\widetilde\Theta_{P^\ast},\widehat\Sigma-\Sigma\right\rangle+ \nonumber\\
 & \left\{ \left\langle \Delta_{P},\Sigma\right\rangle -\left[\log|\widehat\Theta_P| - \log|\widetilde\Theta_P|\right]\right\} +  \left\{\left\langle\Delta_{P^\ast},\Sigma\right\rangle -\left[\log|\widetilde\Theta_{P^\ast}| - \log|\widehat\Theta_{P^\ast}|\right]\right\} \nonumber\\
    = & \left\{ {\ell}^\ast(\widetilde{\Theta}_{P^\ast}) -{\ell}^\ast(\widetilde{\Theta}_{P})\right\} + \left\langle \Delta_{P},\widehat\Sigma - \Sigma\right\rangle
    _{\mathcal{I}_P}+ \left\langle \Delta_{P^\ast},\widehat\Sigma-\Sigma\right\rangle_{\mathcal{I}_{P^\ast}}  + \left\langle \widetilde\Delta_P,\widehat\Sigma-\Sigma\right\rangle_{\mathcal{I}_P\cup\mathcal{I}_{P^\ast}} + \nonumber\\
    &-
      \frac{1}{2}\ve{\Delta_P}^\top(\widetilde\Theta_P + v_1\Delta_P)^{-1}\otimes(\widetilde\Theta_P + v_1\Delta_P)^{-1}\ve{\Delta}\nonumber\\
      & + \frac{1}{2}\ve{\Delta_{P^\ast}}^\top(\widetilde\Theta_{P^\ast} + v_2\Delta_{P^\ast})^{-1}\otimes(\widetilde\Theta_{P^\ast} + v_2\Delta_{P^\ast})^{-1}\ve{\Delta_{P^\ast}}\nonumber\\
    \geq & \left\{\ell^\ast(\widetilde\Theta_{P^\ast}) - \ell^\ast(\widetilde\Theta_{P})\right\}  -\|\Delta_P\|_F\|\widehat\Sigma - \Sigma\|_{F, \mathcal{I}_P}  -\|\Delta_{P^\ast}\|_F\|\widehat\Sigma - \Sigma\|_{F, \mathcal{I}_{P^\ast}}\nonumber\\
      &- \|\widetilde\Delta_P\|_F\|\widehat\Sigma - \Sigma\|_{F, \mathcal{I}_P\cup\mathcal{I}_{P^\ast}} - \frac{\|\Delta_P\|_F^2}{2\gamma_{\min}^2(\widetilde\Theta_P + v_1\Delta_P)}\nonumber\\
      = &  X_P - \widehat{X}_P,\label{eq:proof-thm-lowerbound}
  \end{align}
where
\begin{equation*}
    X_P := \ell^\ast(\widetilde\Theta_{P^\ast}) - \ell^\ast(\widetilde\Theta_{P}),
\end{equation*}
\begin{align*}
    \widehat{X}_P  := & \|\Delta_P\|_F\|\widehat\Sigma - \Sigma\|_{F, \mathcal{I}_P}  +\|\Delta_{P^\ast}\|_F\|\widehat\Sigma - \Sigma\|_{F, \mathcal{I}_{P^\ast}} +\\
    & \|\widetilde\Delta_P\|_F\|\widehat\Sigma - \Sigma\|_{F, \mathcal{I}_P\cup\mathcal{I}_{P^\ast}} + \frac{\|\Delta_P\|_F^2}{2\gamma_{\min}^2(\widetilde\Theta_P + v_1\Delta_P)}.
\end{align*}
Hence, a sufficient condition for exact recovery of $P^\ast$ is that $X_P> \widehat X_P$ for all $P\neq P^\ast$, and thus, the probability that the solution $\widehat P$ is different from the correct permutation $P^\ast$ can be bounded from above as
\begin{align}
    \mathbb{P}\left(\bigcup_{P\neq P^\ast} \left\{\hat{\ell}(\widehat\Theta_P) >\hat{\ell}(\widehat\Theta_{P^\ast})\right\} \right)    \leq  & 
    \mathbb{P}\left(\exists P\neq P^\ast \text{ s.t. }\widehat{X}_P > X_P \right)\nonumber\\
    \leq & 
    \mathbb{P}\left(\max_{P\neq P^\ast}\left\{\widehat{X}_P - X_P\right\} > 0 \right)\nonumber\\
    \leq & \mathbb{P}\left(\max_{P\neq P^\ast}\left\{\|\Delta_P\|_F\|\widehat\Sigma - \Sigma\|_{F, \mathcal{I}_P}\right\} > \frac{t_1}{4}\right) +\nonumber \\
    & \mathbb{P}\left(\left\{\|\Delta_{P^\ast}\|_F\|\widehat\Sigma - \Sigma\|_{F, \mathcal{I}_{P^\ast}}\right\} > \frac{t_1}{4}\right) + \nonumber\\
    &\mathbb{P}\left(\max_{P\neq P^\ast}\left\{\frac{\|\Delta_P\|_F^2}{2\gamma_{\min}^2(\widetilde\Theta_P + v_1\Delta_P)}\right\} > \frac{t_1}{4}\right) + \nonumber \\
    & \mathbb{P}\left(\max_{P\neq P^\ast}\left\{\|\widehat\Sigma - \Sigma\|_{F, \mathcal{I}_P\cup\mathcal{I}_{P^\ast}} - \frac{X_P}{4\|\widetilde\Delta_P\|_F}\right\} > 0\right), \label{eq:proofthm-4termsbound}
\end{align}
where $t_1$ is any positive constant  that satisfies $t_1\leq \min_{P\neq P^\ast} X_P$.

To bound the  terms in Equation~\eqref{eq:proofthm-4termsbound}, we use  Lemmas~\ref{lemmma:covariance_frobenius} and \ref{lemma:theta-error}.
For the first and second terms, observe that
\begin{align}
    \mathbb{P}\left(\max_{P\in\Pi_n}\|\Delta_P\|_F\|\widehat\Sigma - \Sigma\|_{F, \mathcal{I}_P} > \frac{t_1}{4}\right) & \leq
    \mathbb{P}\left(\left\{\max_{P\in\Pi_n}\|\Delta_P\|_F > \frac{\sqrt{t_1}}{2}\right\} \cup \left\{\max_{P\in\Pi_n}\|\widehat\Sigma - \Sigma\|_{F, \mathcal{I}_P} > \frac{\sqrt{t_1}}{2}\right\}\right)\nonumber\\
    &\leq
    \mathbb{P}\left(\max_{P\in\Pi_n}\|\Delta_P\|_F^2 > \frac{t_1}{4}\right) + \mathbb{P}\left(\max_{P\in\Pi_n}\|\widehat\Sigma - \Sigma\|_{F, \mathcal{I}_P}^2 > \frac{t_1}{4}\right)\nonumber\\
    & \leq 2c_1\exp\left(2\log n  - 
    \frac{c_2 mt_1}{\max\{32\kappa_2, 4\}(\|A\|_F^2 + n)} \right), \label{eq:proof-thm-Delta-P-Sigma}
\end{align}
for any $t_1\leq T_1:=4\min\left\{ \frac{1}{8\kappa_2}, 8\kappa_2\delta (\|A\|_F^2 +n), \delta (\|A\|_F^2 +n) \right\}$.
To obtain an upper bound for the third term of Equation~\eqref{eq:proofthm-4termsbound},  notice that
\begin{align*}
    \gamma_{\min}(\widetilde\Theta_P + v_1\Delta_P) & \geq \gamma_{\min}(\widetilde\Theta_P) - v_1\gamma_{\max}(\Delta_P)\\
    & \geq \gamma_{\min}(\widetilde\Theta_P) - \|\Delta_P\|_F.
\end{align*}
Hence, under the event $\|\Delta_P\|_F\leq \tau \gamma_{\min}(\widetilde\Theta_P)$ for some $\tau\in(0,1)$,
$$\frac{\|\Delta_P\|_F^2}{2\gamma_{\min}^2(\widetilde\Theta_P + v_1\Delta_P)}\leq \frac{\|\Delta_P\|_F^2}{2(1-\tau)^2\gamma_{\min}^2(\widetilde\Theta_P)} \leq \frac{\|\Delta_P\|_F^2}{2(1-\tau)^2\kappa_3^2}.$$
Therefore, 
\begin{align}
\mathbb{P}\left(\max_{P\in\Pi_n}\left\{\frac{\|\Delta_P\|_F^2}{2\gamma_{\min}^2(\widetilde\Theta_P + v_1\Delta_P)}\right\}\geq \frac{t_1}{4}\right) & \leq \mathbb{P}\left(\left\{\max_{P\in\Pi_n}\frac{\|\Delta_P\|_F^2}{2(1-\tau)^2\kappa_3^2}\geq \frac{t_1}{4}\right\} \cup \left\{\max_{P\in\Pi_n}\frac{\|\Delta_P\|_F^2}{\kappa_3^2}>\tau^2\right\}\right) \nonumber\\
    & \leq  \mathbb{P}\left(\max_{P\in\Pi_n}\|\Delta_P\|_F^2 \geq (1-\tau)^2\kappa_3^2\frac{t_1}{2} \right) + \mathbb{P}\left( \max_{P\in\Pi_n}\|\Delta_P\|_F^2 \geq \tau^2\kappa_3^2\right)\nonumber \\
    & \leq 2\mathbb{P}\left(\max_{P\in\Pi_n}\|\Delta_P\|_F^2 \geq \kappa_3^2\min\left\{ (1-\tau)^2\frac{t_1}{2}, \tau^2\right\}\right)\nonumber \\ 
    & \leq 2c_1\exp\left(2\log n - \frac{c_2' m \kappa_3^2t_1}{32\kappa_2(\|A\|_F^2 + n)} \right),\label{eq:proof-thm-loweboundkronecker}
\end{align}
where $c_2'= 2(1-\tau)^2c_2$, $t_1\leq T_2 :=  (2(1-\tau)^2\kappa_3^2)^{-1} 4\min\left\{\frac{1}{8\kappa_2}, 8\kappa_2\delta (\|A\|_F^2 +n) \right\}$, and $\tau$ is a constant such that $t_1\leq \frac{2\tau^2}{(1-\tau)^2}$.

Finally, to obtain an upper bound for the last term of Equation~\eqref{eq:proofthm-4termsbound}, observe  that the value of $X_P$ can be bounded from below using  Equations~\eqref{eq:taylor-logdet}, \eqref{eq:kronecker-lowerbound}, and the fact that $\nabla\ell^\ast(\widetilde\Theta_{P^\ast}) = 0$, which  ensure  there exists a constant  $\tilde{v}\in[0,1]$ such that
\begin{align*}
   X_P & = \frac{1}{2}\ve{\widetilde\Delta_P}^\top(\widetilde\Theta_{P^\ast} + \tilde{v}\widetilde{\Delta}_P)^{-1}\otimes(\widetilde\Theta_{P^\ast} + \tilde{v}\widetilde{\Delta}_P)^{-1}\ve{\widetilde\Delta_P} \nonumber\\
    & \geq \frac{\|\widetilde\Delta_P\|_F^2}{2\gamma_{\max}^2(\widetilde\Theta_{P^\ast} + \tilde v\widetilde\Delta_P)}\nonumber\\
    %& \geq \frac{\|\widetilde\Delta_P\|_F^2}{2((1-\tilde v)\|\widetilde\Theta_{P^\ast}\| + \tilde v\|\widetilde\Theta_P\|)^2}\nonumber\\
    %& \geq \frac{\|\widetilde\Delta_P\|_F^2}{2\max\{\|\widetilde\Theta_{P^\ast}\|,\|\widetilde\Theta_P\|\}^2}\nonumber \\
    & \geq \frac{\|\widetilde\Delta_P\|_F^2}{2\kappa_2^2}.
\end{align*}
Therefore, using the previous bound and Lemma~\ref{lemmma:covariance_frobenius}, 
\begin{align}
    \mathbb{P}\left(\max_{P\neq P^\ast}\left\{\|\widehat\Sigma - \Sigma\|_{F, \mathcal{I}_P\cup\mathcal{I}_{P^\ast}} - \frac{X_P}{4\|\widetilde\Delta_P\|_F}\right\} > 0\right) & \leq 
    \mathbb{P}\left(\max_{P\neq P^\ast}\left\{\|\widehat\Sigma - \Sigma\|_{F, \mathcal{I}_P\cup\mathcal{I}_{P^\ast}} - \frac{\|\widetilde\Delta_P\|_F}{8\kappa_2^2}\right\} > 0\right) \nonumber \\
    & \leq 
  \mathbb{P}\left(\max_{P\neq P^\ast}\|\widehat\Sigma - \Sigma\|_{F, \mathcal{I}_P\cup\mathcal{I}_{P^\ast}}^2   > \min_{P\neq P}\frac{\|\widetilde\Delta_P\|_F^2}{64\kappa_2^4}\right) \nonumber\\
& \leq c_1\exp\left(2\log n - \frac{c_2mt_2}{(2\|A\|_F^2 + n)}\right),\label{eq:proofthm-sigmaXp-bound}
\end{align}
for any $t_2\leq T_3:=\min\left\{\frac{\|\widetilde\Delta_P\|_F^2}{64\kappa_2^4}, \delta(2\|A\|_F^2+n)\right\}$, where we used the fact that $|\mathcal{I}_P\cup\mathcal{I}_{P^\ast}|\leq 2\|A\|_F^2 + n$.

Therefore, setting 
$$t_1 = t_2 = \left(\frac{24\max\{1, 8\kappa_2\}}{\min\{c_2, c_2'\kappa_3^2\}}\right) \frac{(\|A\|_F^2+ n)\log n}{m},$$
observe that the condition in Equation~\eqref{eq:thm-condition} of the Theorem guarantee that $t_1\leq T_1\wedge T_2$ and $t_2\leq T_3$, 
and combining Equations~\eqref{eq:proof-thm-Delta-P-Sigma}, \eqref{eq:proof-thm-loweboundkronecker} and \eqref{eq:proofthm-sigmaXp-bound} into Equation~\eqref{eq:proofthm-4termsbound}, we obtain that
\begin{equation*}
    \mathbb{P}\left(\hat{\ell}(\widehat\Theta_{P^\ast}) > \hat{\ell}(\widehat\Theta_P)\ \forall P\in\Pi_n\setminus\{P^\ast\}\right) \geq 1- \frac{c_1}{n}.
\end{equation*}

\end{proof}

%%%%%%%%%%%%%%%%%%%%%%%%%%%%%%%%%%%%%%%%%%%%%%%%%%%%%%%%%%%%%%%%%%%%%%%%%%%%%%%%%%%%%%%%%%%%%%%%%
%%%%%%%%%%%%%%%%%%%%%%%%%%%%%%%%%%%%%%%%%%%%%%%%%%%%%%%%%%%%%%%%%%%%%%%%%%%%%%%%%%%%%%%%%%%%%%%%%
%%%%%%%%%%%%%%%%%%%%%%%%%%%%%%%%%%%%%%%%%%%%%%%%%%%%%%%%%%%%%%%%%%%%%%%%%%%%%%%%%%%%%%%%%%%%%%%%%
\section{Proof of Theorem 2}
\begin{proof}
Under the assumptions of the theorem, the expression of the log-likelihood can be simplified via
\begin{align*}
    \hat \ell_{m}(\theta, P) & = \frac{1}{m} \sum_{k=1}^m \left[ \theta\sum_{i \neq j} (P^\top AP)_{ij}B_{ik}B_{ij}- \log Z(\theta)\right]\\
    & = \theta \text{Tr}\left(P^\top AP \tilde B\right)  - \Psi(\theta),
\end{align*}
where $\Psi(\theta):=\log Z(\theta) =\log\left(\sum_{u\in\mathcal{U}} \exp(\theta u)\right)$ is the corresponding log-partition function with $\mathcal{U}=\{y^\top Ay:y\in\{0,1\}^n\}$. Observe that $\Psi(\theta)$ does not depend on $P$.

%Additionally, $\Psi(\theta_0)$ is an increasing function of $\theta$.

Define $\hat{\theta}_P$ as the profile MLE of $\theta$ given  $P\in\Pi_n$, i.e.,
$$\hat{\theta}_P = \argmax_{\theta\in\mathbb{R}}\hat{\ell}(\theta,P).$$
By taking a derivative of the log-likelihood and equating to zero, the profile MLE satisfies
\begin{align}
\label{eq:GM-MLE-deriv}
    \text{Tr}\left(P^\top AP \tilde B\right)  & = \left.\frac{d\Psi(\theta)}{d\theta}\right|_{\theta = \hat{\theta}_P}\\
    & = \frac{\sum_{u\in\mathcal{U}} u\exp(u\hat{\theta}_P )}{Z(\hat{\theta}_P)}.\notag
\end{align}
The last equation is an increasing function of $\hat{\theta}_P$. To observe that, we show that the second derivative of $\Psi(\theta)$ is positive. This is given by
\begin{align*}
    \frac{d^2\Psi(\theta)}{d\theta^2} & = \frac{Z(\theta)\left[\sum_{u\in\mathcal{U}}u^2\exp(u\theta)\right] - \left(\sum_{u\in\mathcal{U}}u\exp(u\theta)\right)\left(\sum_{v\in\mathcal{U}}v\exp(v\theta)\right)} {(Z(\theta))^2}\\
    & = \frac{\sum_{u\in\mathcal{U}}\sum_{v\in\mathcal{U}}\left(u^2\exp(u\theta)\exp(v\theta) - uv\exp(u\theta)\exp(v\theta)\right)} {(Z(\theta))^2}\\
     %& = \frac{\frac{1}{2}\sum_{u,v}u^2\exp(u\theta)\exp(v\theta) +\frac{1}{2}\sum_{u,v}v^2\exp(u\theta)\exp(v\theta)-
    % \sum_{u,v}uv\exp(u\theta)\exp(v\theta)} {(Z(\theta))^2}\\
    & = \frac{\sum_{u,v}\frac{1}{2}\left(u-v\right)^2[\exp(u\theta)\exp(v\theta)]} {(Z(\theta))^2}\geq 0,
\end{align*}
and it is only equal to zero when $u=v$ for all $u,v\in\mathcal{U}$, which implies that $A=0$. Hence, the second derivative of $\Psi(\theta)$ is strictly positive.

Using  Equation \eqref{eq:GM-MLE-deriv}, the MLE for $P$ is
\begin{align}
    \hat{P}_{\text{MLE}} & = \argmax_{P\in\Pi_n}\hat{\ell}(\hat{\theta}_P, P) \notag\\
    & = \argmax_{P\in\Pi_n}\left\{ \hat{\theta}_P \left.\frac{d\Psi(\theta)}{d\theta}\right|_{\theta = \hat{\theta}_P} - \Psi(\hat{\theta}_P) \right\}. \label{eq:PMLE}
\end{align}
Observe that the previous objective function only depends on $\hat{\theta}_P$. 
Moreover, the derivative of this function with respect to $\theta$ and evaluated at $\hat{\theta}_P$ is equal to
$$
\hat{\theta}_P\left.\frac{d^2\Psi(\theta)}{d\theta^2}\right|_{\theta = \hat{\theta}_P}.
$$
When $\hat{\theta}_{\text{MLE}}=\hat{\theta}_{\hat P_{\text{MLE}}}$ is positive, this derivative is greater than $0$, which means that the objective function $\hat{\ell}(\hat{\theta}_P, P)$ in Equation~\eqref{eq:PMLE} is increasing as a function of $\hat{\Theta}_P$, 
and $\hat P_{\text{MLE}}$ is chosen via 
\begin{equation}
    \hat{P}_{\text{MLE}} = \argmax_{P\in\Pi_n} \text{Tr}\left(P^\top AP\tilde{B}\right) = \argmin_{P\in\Pi_n}\|A-P\tilde{B}P^\top \|_F^2.\label{eq:mle-profile}
\end{equation}
To see this, consider the existence of a $P$ such that 
$$\text{Tr}\left(P^\top AP\tilde{B}\right)>\text{Tr}\left(\hat P_{\text{MLE}}^\top A\hat P_{\text{MLE}}\tilde{B}\right).$$
Then from Equation~\eqref{eq:GM-MLE-deriv}, $0<\hat{\theta}_{\text{MLE}}<\hat{\theta}_{P}$ and because the objective function in \eqref{eq:PMLE} is increasing as a function of $\hat{\theta}_{P}$, then we have that
$$\hat{\ell}(\hat{\theta}_P, P)> \hat{\ell}(\hat{\theta}_{\text{MLE}}, \hat P_{\text{MLE}})$$
yielding a contradiction.
Analogously, when $\hat{\theta}_P$ is negative, the derivative of Equation~\eqref{eq:PMLE} is negative, and hence the MLE for $P$ is obtained via
 \begin{equation}
    \hat{P}_{\text{MLE}} = \argmin_{P\in\Pi_n} \text{Tr}\left(P^\top AP\tilde{B}\right) = \argmax_{P\in\Pi_n}\|A-P\tilde{B}P^\top \|_F^2.
\end{equation}
Finally, if $\hat{\theta}_{\text{MLE}}=0$, then $\hat\theta_P=0$ for all $P$ (by considering the contradiction present for $\hat\theta_P\neq 0$ from Equation~\eqref{eq:GM-MLE-deriv}).
Therefore, any $P$ achieves the optimum value in Equation~\eqref{eq:mle-profile}.
 \end{proof}